\documentclass[twoside]{article}

\usepackage{amssymb,amsmath,color}
\usepackage{dsfont}
\usepackage{url}
\usepackage{pifont}

\usepackage{stmaryrd}

\newcommand{\NRM}[1]{{{\left\| #1\right\|}}} 
\newcommand{\esp}[1]{\mathbb{E}\left[#1\right]}

\usepackage{graphicx}
\usepackage{subfigure}

\usepackage{algorithm}
\usepackage{algorithmic}

\usepackage[round]{natbib}

\usepackage{hyperref}

\hypersetup{
  colorlinks   = true, 
  urlcolor     = cyan, 
  linkcolor    = blue, 
  citecolor    = blue, 
}

\usepackage{amsthm}
\theoremstyle{plain}
\newtheorem{lemma}{Lemma}

\newtheorem{proposition}{Proposition}
\newtheorem{corollary}{Corollary}
\theoremstyle{definition}
\newtheorem{definition}{Definition}
\newtheorem{assumption}{Assumption}
\newtheorem{remark}{Remark}
\newtheorem{example}{Example}
\newtheorem{property}{Property}

\newcommand{\Sec}[1]{Section~\ref{sec:#1}}
\newcommand{\Appendix}[1]{Appendix~\ref{appendix:#1}}
\newcommand{\Eq}[1]{Eq.~(\ref{eq:#1})}
\newcommand{\Alg}[1]{Alg.~(\ref{alg:#1})}
\newcommand{\Tab}[1]{Table~\ref{tab:#1}}

\newcommand{\Corollary}[1]{Corollary~\ref{cor:#1}}
\newcommand{\Lemma}[1]{Lemma~\ref{lem:#1}}
\newcommand{\Example}[1]{Example~\ref{ex:#1}}
\newcommand{\Proposition}[1]{Proposition~\ref{prop:#1}}
\newcommand{\Definition}[1]{Definition~\ref{def:#1}}
\newcommand{\Assumption}[1]{Assumption~\ref{ass:#1}}

\newcommand{\BEAS}{\begin{eqnarray*}}
\newcommand{\EEAS}{\end{eqnarray*}}
\newcommand{\BEA}{\begin{eqnarray}}
\newcommand{\EEA}{\end{eqnarray}}
\newcommand{\BAL}{\begin{align}}
\newcommand{\EAL}{\end{align}}
\newcommand{\BEQ}{\begin{equation}}
\newcommand{\EEQ}{\end{equation}}
\newcommand{\BEQe}{\begin{equation*}}
\newcommand{\EEQe}{\end{equation*}}

\newcommand{\BIT}{\begin{itemize}}
\newcommand{\EIT}{\end{itemize}}
\newcommand{\BNUM}{\begin{enumerate}}
\newcommand{\ENUM}{\end{enumerate}}
\newcommand{\BA}{\begin{array}}
\newcommand{\EA}{\end{array}}

\newcommand{\one}{\mathds{1}}
\newcommand{\scprod}[2]{\langle#1,#2\rangle}

\newcommand{\Var}{\mathop{ \rm var{}}}
\newcommand{\argmin}{\mathop{\rm argmin}}

\newcommand{\var}{\mathop{ \rm var}}

\def \bF{{\mathbb F}}

\def \bM{{\mathbb M}}
\def \cA{{\mathcal A}}
\def \cB{{\mathcal B}}
\def \cD{{\mathcal D}}
\def \cH{{\mathcal H}}
\def \cG{{\mathcal G}}

\def \cO{{\mathcal O}}

\def \cL{{\mathcal L}}
\def \cX{{\mathcal X}}
\def \cY{{\mathcal Y}}
\def \cZ{{\mathcal Z}}

\def \cF{{\mathcal F}}
\def \cR{{\mathcal R}}
\def \sE{{\mathsf E}}
\def \sV{{\mathsf V}}
\def \sA{{\mathsf A}}
\def \sO{{\mathsf O}}
\def \E{{\mathbb E}}
\def \P{{\mathbb P}}
\def \R{{\mathbb R}}
\def \N{{\mathbb N}}

\newcommand{\Exp}[1]{\E\left[#1\right]}
\newcommand{\ExpUnder}[2]{\E_{#1}\left[#2\right]}
\newcommand{\Prob}[1]{\P\left(#1\right)}
\newcommand{\ProbUnder}[2]{\P_{#1}\left(#2\right)}

\newcommand{\set}[1]{\llbracket#1\rrbracket}

\newcommand{\ie}{i.e.\ }
\newcommand{\eg}{e.g.\ }
\newcommand{\wrt}{w.r.t.\ }
\newcommand{\st}{\mbox{\ s.t.\ }}
\newcommand{\iid}{i.i.d.\ }

\newcommand{\Loja}{\L{}ojasiewicz}

\newcommand{\condnum}[1]{\kappa}

\newcommand{\ub}{\sigma_{2,\cG}\left( \sE_\cD | \sO \right)}
\newcommand{\lb}{\sigma_{\cG,2}\left( \sE_\cD | \sO \right)}
\newcommand{\mmratesc}[1]{\varepsilon_{\mbox{\rm\tiny sc}}(#1)}
\newcommand{\mmratepl}[1]{\varepsilon_{\mbox{\rm\tiny pl}}(#1)}
\newcommand{\sON}[1]{\sO^{\mbox{\rm\tiny #1}}_n}

\newcommand{\cGN}[1]{\cG_{\mbox{\rm\tiny #1}}}
\newcommand{\dN}[1]{d_{\mbox{\rm\tiny #1}}}

\newcommand{\Arand}{\cA_{\mbox{\rm\tiny rand}}}
\newcommand{\SV}[2]{\|\sV_{#1}\|_{#2}}

\usepackage[accepted]{aistats2024new}
\begin{document}

%
\runningtitle{Minimax Excess Risk of First-Order Methods with Data-Dependent Oracles}

%

\twocolumn[

\aistatstitle{Minimax Excess Risk of First-Order Methods\\for Statistical Learning with Data-Dependent Oracles}

\aistatsauthor{ Kevin Scaman \And Mathieu Even \And  Batiste Le Bars \And Laurent Massoulié }


\aistatsaddress{ Inria Paris - Département d’informatique de l’ENS, PSL Research University } ]

\begin{abstract}
In this paper, our aim is to analyse the generalization capabilities of first-order methods for statistical learning in multiple, different yet related, scenarios including supervised learning, transfer learning, robust learning and federated learning.
To do so, we provide sharp upper and lower bounds for the minimax excess risk of strongly convex and smooth statistical learning when the gradient is accessed through partial observations given by a \emph{data-dependent} oracle. This novel class of oracles can query the gradient with any given data distribution, and is thus well suited to scenarios in which the training data distribution does not match the target (or test) distribution.
In particular, our upper and lower bounds are proportional to the smallest mean square error achievable by gradient estimators, thus allowing us to easily derive multiple sharp bounds in the aforementioned scenarios using the extensive literature on parameter estimation.
\end{abstract}

\section{INTRODUCTION}

In statistical learning, one is often interested in minimizing a population risk, also known as test loss, of the form $\cL(x)=\E_{\xi\sim\cD}[\ell(x,\xi)]$ for some loss function $\ell$ and (unknown) test data distribution $\cD$. The question that arises then is how small can the \emph{excess risk} $\cL(\hat x)-\inf_x \cL(x)$ be, for $\hat x$ computed using some given restricted information?

In classical supervised learning settings, $\hat x$ is typically computed with $n$ \iid samples drawn from $\cD$ and usually corresponds to the minimizer of the empirical counterpart of $\cL(x)$ computed with those samples. In that setting, the excess risk can be controlled through the concept of \emph{generalization error}, which quantifies the degree to which minimizing the empirical risk, also known as the training loss, is similar to minimizing the test loss. Among the several approaches that have been proposed to bound generalization errors, the most prominent ones are based on the complexity of the hypothesis class like the Vapnik-Chervonenkis dimension \citep{Vapnik2000,Vapnik2015,blumer1989learnability} or Rademacher complexity \citep{bartlett2003rademacher,Bousquet2004}, algorithmic stability \citep{Mukherjee2006,bousquet2002stability}, PAC-Bayesian bounds \citep{mcallester98, alquier2021user}, or more recently information-theoretic generalization bounds \citep{xu2017information}. Nowadays, it is commonly accepted that, in this context, the generalization error alone is not sufficient to control the excess risk. As the empirical risk minimizer cannot always be computed in an exact manner, the \emph{optimization error} must also be taken into account, measuring the algorithm's ability to properly minimize the empirical risk, and shedding light on a generalization-optimization trade-off \citep{NIPS2007_0d3180d6}. Over the last few years, a substantial amount of work have therefore been dedicated in controlling these errors, notably through the study of the generalization properties of \emph{optimization algorithms} \citep{lin2016generalization, london2017pac, zhou2018generalization, amir2021sgd, neu2021information}, where approaches based on algorithmic stability have encountered a large success \citep{hardt2016train,kuzborskij2018data,bassily2020stability,lei2020fine, lei2020sharper, schliserman2022stability}.

Above approaches are however mostly tailored for standard supervised learning and empirical risk minimization. Hence, an additional analysis is required for all the different variations and flavors of this problem, such as transfer learning / domain adaptation \citep{bendavid2006domain} in which the training distribution differs from that of the testing distribution, or robust learning in which a small portion of the training data may be corrupted by an arbitrary noise. Note further that, while there exist a large panel of upper bounds on the generalization error, the optimization error \citep{Arjevani2023,bubeck2015convex,drori2022oracle} or, more generally, the excess risk, the question of their optimality with respect to some lower bounds is most of the time lacking or specific to a particular algorithm or class of data-distribution \citep{zhang2022stability,schliserman2023tight}. For instance, \citet{Arjevani2023} proved lower bounds for the non-convex stochastic case, under an oracle framework that inspired our formalism; \citet{Devolder2013inexact} considers \textit{inexact oracles}. However, \citet{Devolder2013inexact}'s work is quite different from ours, since their goal is to analyze different algorithms under a unified framework. \looseness = -1

\paragraph{Contributions.} In this work, we propose a unified framework to analyse, among others, the aforementioned statistical learning problems in a more systematic manner.
Our general framework goes beyond the decomposition between generalization and optimization error and analyzes instead directly the ability of an optimization algorithm to minimize the population risk given partial, and possibly biased, information.
Our contributions can be summarized as follows:

\begin{itemize}
    \item We tackle the problem of controlling the test loss $\cL$ trough the lens of first-order optimization methods with gradient oracles. Contrary to the stochastic optimization setting \citep{agarwal_information-theoretic_2011, Arjevani2023}, we introduce the novel notion of \textbf{data-dependent oracle}, more adapted to the case where the gradients are computed over a fixed data set, used possibly several times during optimization. 
    \item The data used by the oracle being arbitrary, we show that \textbf{our setting is rather generic} and contains supervised learning, transfer learning, robust learning, and federated learning. 
    \item We provide \textbf{upper and lower bounds} for the minimax excess risk of statistical learning problems and show that they are sharp for deterministic and for more classical \iid oracles. Our bounds shed light on a novel quantity called \textbf{best approximation error}, generalizing conditional expectations and conditional standard deviations.
    \item We show that our general bounds can be applied to \textbf{several learning settings}, allowing to obtain problem-specific excess risk bounds and recover some known results of the literature. In particular, we show that in the case of standard supervised learning, mini-batch gradient descent with increasing batch sizes and a warm start can reach an excess risk that is optimal up to a multiplicative factor, thus motivating the use of this optimization scheme in practical applications.
\end{itemize}

\paragraph{Outline of the paper.} In \Sec{problem_setup} we introduce our statistical learning setting, where we define data-dependent oracles, the algorithms considered, as well as our set of assumptions.
We also introduce the aforementioned quantity called \emph{best approximation error}. 
In \Sec{gen_bounds}, we derive upper and lower bounds for the minimax excess risk of statistical learning with any given data-dependant oracle and discuss their optimality.
Finally, in \Sec{applications} we apply our general bounds to supervised learning, transfer learning, federated learning, robust learning and learning from a fixed predetermined dataset.

\paragraph{Notations.}
In what follows, we denote as $\bF(\cX,\cY)$ (resp. $\bM(\cX,\cY)$) the space of functions (resp. measurable functions) from $\cX$ to $\cY$ (both measurable spaces).
Let $\|x\| = \sqrt{\sum_i x_i^2}$ be the canonical norm in $\R^d$, and $\rho(A)$ the nuclear norm of the matrix $A\in\R^{d\times D}$.
A function $f$ is $B$-Lipschitz if $\|f(x) - f(x')\| \leq B \|x - x'\|$ for all $x,x'\in\cX$.
A differentiable function $f:\R^d\to\R$ is $\mu$-strongly convex (where $\mu\geq0$) if $\forall x,y\in\R^d$, we have $f(x)-f(y)\geq \langle \nabla f(y),x-y\rangle +\frac{\mu}{2}\NRM{x-y}^2$, and convex if this holds for $\mu=0$.
$f$ is $L$-smooth if it is differentiable and its gradient is $L$-Lipschitz.
Finally, for two functions $a,b:\cZ\to \R^+$, we write $a=\Theta(b)$ (resp. $a=O(b)$) if there exists $c,C>0$ such that for all $z\in\cZ$, $cb(z)\leq a(z)\leq Cb(z)$ (resp. $a(z)\leq Cb(z)$).

\section{PROBLEM SETUP}\label{sec:problem_setup}

We now provide precise definitions for statistical learning under data-dependent oracles, as well as the minimax estimation error used in our analysis.

\subsection{Statistical learning}
\label{sec:stat_learn}
Consider the population risk minimization problem:
\BEQ\label{eq:stat_learn}
\inf_{x\in\R^d} \cL(x) \triangleq \ExpUnder{\xi\sim\cD}{\ell(x, \xi)}\,,
\EEQ
where $\cD$ is a probability distribution over the measurable space $\Xi$ and $\ell:\R^d\times \Xi\to\R$ is a loss function that takes as input a model parameter $x\in\R^d$ and a data point $\xi\in\Xi$.
For simplicity, for any $y\in\R^d$, we denote as $\nabla\ell_y:\xi\mapsto\nabla_x\ell(y,\xi)$ the gradient of the loss \wrt its first coordinate.
Moreover, our analysis focuses on strongly-convex and smooth objective functions whose gradients belong to a given function class.

\begin{definition}[function class]
    Let $\cG\subset\bF(\Xi,\R^d)$ be a class of functions taking data points as input. We denote as $\cF_{\mbox{\tiny sc}}(\cG,\cD, \mu, L)$ the set of $\mu$-strongly convex and $L$-smooth objective functions $\cL(x) = \ExpUnder{\xi\sim\cD}{\ell(x,\xi)}$ such that $\forall x\in\R^d$, $\nabla\ell_x \in \cG$.
\end{definition}

The function class $\cG$ is used to encode the regularity of the gradient of the loss with respect to input data, for example (assuming $\Xi=\R^D$ for the first two):
1) affine functions: $\cGN{Aff} = \{ \xi\mapsto A\xi + b :A\in\R^{d\times D},b\in\R^d, \rho(A)\leq B\}$,
2) Lipschitz functions: $\cGN{Lip}=\{ g:\Xi\to\R^d :\forall \xi,\xi'\in\Xi, \|g(\xi) - g(\xi')\| \leq B\|\xi - \xi'\| \}$, and
3) bounded variations: $\cGN{Bnd}=\{ g:\Xi\to\R^d :\exists c_g\in\R^d,\forall \xi\in\Xi, \|g(\xi) - c_g\| \leq B \}$.
Note that all these function spaces are invariant by translation by a constant, a key property for our analysis (see \Assumption{O}).

\begin{remark}
All our results, lower and upper bounds, also apply to the minimax excess risk of non-convex smooth and $\mu$-PL functions (see \Appendix{PL}).
\end{remark}

\begin{example}[Least squares regression]
Let $\ell(x,\xi=(M,v))=\frac{1}{2}x^\top M x - v^\top x$ for $M\in\R^{d\times d}$ and $v\in\R^d$, leading to $\cG=\cGN{Aff}$ where $B$ is the diameter of the space over which we optimize. 
\end{example} 

\begin{example}[Regularized Lipschitz losses]\label{ex:reglip}
Let $\cL(x)= \lambda\Omega(x) + \E_\cD[\ell(x,\xi)]$, for some Lipschitz continuous and convex loss $\ell$ (in its first argument) and a convex regularizer $\Omega$, yielding $\cG=\cGN{Bnd}$.
\end{example} 

\subsection{Data-dependent oracles and first-order optimization algorithms}
Our objective is to minimize \Eq{stat_learn} using optimization algorithms that access $\nabla\cL$ via a \emph{data-dependent oracle}, in a setup similar to that of \citet{Arjevani2023}.

\begin{definition}[Data-dependent oracle]
Let $\cO,\cZ$ be two measurable spaces and $\bF(\Xi,\R^d)$ a measurable space of functions.
A \emph{data-dependent oracle} is a tuple $(\sO,P_z)$ where $\sO: \bF(\Xi,\R^d)\times\cZ\to\cO$ is a measurable function and $P_z$ is a probability distribution over $\cZ$. 
\end{definition}

At each iteration, optimization algorithms will only be able to access the gradient of the objective function through the \emph{observation} $\sO(\nabla\ell_x,z)$, where $x\in\R^d$ is the current model parameter and $z\sim P_z$ is a random seed drawn \emph{prior} to the optimization.
In other words, an oracle provides a partial (and possibly random) view of the gradient, for example by accessing the gradient at \iid sampled data points $\xi_i'\sim\cD'$ drawn according to a source data distribution $\cD'\neq\cD$. In such a case, we have $\cO = \R^{d\times n}$, $\cZ = \Xi^n$, and $\sO(g,(\xi_1',\dots,\xi_n')) = (g(\xi_1'),\dots,g(\xi_n'))$. Note that, contrary to the online setting of \citet{Arjevani2023}, the randomness is fixed prior to the optimization, and thus each iteration of the optimization will have access to the same data points $\xi_1,\dots,\xi_n$. 
We now define more precisely the class of algorithms that we will consider in this analysis.

\begin{definition}[Optimization algorithm]\label{def:alg}
Let $\cO,\cR$ be two measurable spaces. An \emph{optimization algorithm} is a tuple $\sA=(\{q^{(t)}, s^{(t)}\}_{t\geq 0},P_r)$ where $q^{(t)}\in\bM(\cO^t\times\cR,\R^d)$ is a query function, $s^{(t)}\in\bM(\cO^t\times\cR,\{0,1\})$ is a stopping criterion, and $P_r$ is a distribution over $\cR$.
\end{definition}

For a given data-dependent oracle $(\sO,P_z)$ and optimization algorithm $\sA=(\{q^{(t)}, s^{(t)}\}_{t\geq 0},P_r)$, we consider the following optimization protocol:
\BNUM
\item We first draw two random seeds: $r\sim P_r$ for the algorithm, and $z\sim P_z$ for the oracle.
\item At each iteration $t \geq 0$, we update the iterates:
\BEQ
\BA{lll}
x_{\sA[\sO]}^{(t)} &=& q^{(t)}\left( m_{\sA[\sO]}^{(t)}, r \right)\\
s_{\sA[\sO]}^{(t)} &=& s^{(t)}\left( m_{\sA[\sO]}^{(t)}, r \right)\\
\EA
\EEQ
where $m_{\sA[\sO]}^{(t)} = (\sO( \nabla\ell_{x_{\sA[\sO]}^{(0)}}, z ),\dots,\sO( \nabla\ell_{x_{\sA[\sO]}^{(t-1)}}, z ))$.
\item The algorithm stops and returns the current iterate $x_{\sA[\sO]} = x_{\sA[\sO]}^{(t)}$ as soon as $s_{\sA[\sO]}^{(t)} = 1$.
\ENUM

In other words, at each iteration, the algorithm updates the model parameter based on all past observations, and then decides to stop (and return the current model parameter) or continue the optimization. If so, the algorithm receives a new observation of the gradient for the current model parameter and proceeds to the next iteration.
Note that the algorithm may not terminate, in which case we consider the loss as infinite.
Moreover, as discussed in \citet{Arjevani2023}, fixing the randomness to a single seed $r$ instead of drawing random seeds $r^{(t)}$ for each iteration does not lose any generality.
Finally, we denote as $\Arand$ the class of all optimization algorithms as defined above and, in order to prove lower bounds on the error of optimization algorithms, we assume that the information extracted by the oracle is \emph{invariant} with respect to translations in the following sense.

\begin{assumption}[Translation invariance]\label{ass:O}
There exists a measurable function $\varphi:\cG\times\R^d\to\cO$ such that, for any function $g\in\cG$ and constant $c\in\R^d$, $g+c\in\cG$ and $\forall z\in\cZ$, $\sO(g+c, z) = \varphi(\sO(g, z),c)$.
\end{assumption}

Intuitively, \Assumption{O} means that translations do not add any information to the oracle, as the translated oracles $\sO(g+c, z)$ can be retrieved as a function of the untranslated oracle $\sO(g, z)$. This assumption is verified in most settings of interest (see \Sec{applications}).

\subsection{Minimax excess risk}
We evaluate the difficulty of optimizing functions in $\cF_{\mbox{\tiny sc}}(\cG,\cD,\mu,L)$ (abbreviated to $\cF_{\mbox{\tiny sc}}$ below) with a given oracle $\sO$ via the \emph{minimax excess risk} defined by
\BEQe
\varepsilon_{\mbox{\tiny sc}}(\cG,\sO,\cD,\mu,L) = \inf_{\sA\in\Arand} \sup_{\cL\in\cF_{\mbox{\tiny sc}}} \Exp{\cL\left(x_{\sA[\sO]}\right) - \cL^*}\,,
\EEQe
where $\cL^* = \inf_{x\in\R^d} \cL(x)$ is the minimum value of the objective function.
In other words, the minimax excess risk measures the best worst-case error that a first-order optimization algorithm can achieve on the objective function $\cL$, despite only accessing to the gradients via the oracle $\sO$.
For simplicity, as the terms $\cD$, $\mu$, $L$ will be fixed throughout the paper, we will from now on omit them and only write $\mmratesc{\cG,\sO}$.

\subsection{Minimax estimation error}
\label{sec:cond_std}
Our upper and lower bounds on the minimax excess risk will depend on the ability to create estimators of the expectation over $\cD$ of any function in $\cG$. This notion, denoted as \emph{minimax estimation error}, is defined via \emph{best approximation errors}, a novel notion that extends conditional standard deviation to measurable functions equipped with arbitrary semi-norms.

\begin{definition}[Best approximation error]
Let $\cX$ and $\cY$ be two measurable spaces, $\cZ$ a measurable vector space, and $\|\cdot\|_\nu$ a (possibly infinite) semi-norm over $\bM(\cX,\cZ)$. For $f\in\bM(\cX,\cZ)$ and $h\in\bM(\cX,\cY)$ two measurable functions, we denote as \emph{best approximation error} of $f$ knowing $h$ the quantity
\BEQ\label{eq:best_approx}
\sigma_\nu(f|h) \,=\, \inf_{\varphi\in\bM(\cY,\cZ)}\|f - \varphi\circ h\|_\nu\,.
\EEQ
\end{definition}
In other words, $\sigma_\nu(f|h)$ 
measures how well can $f$ be approximated using $g$, and is thus tightly connected to estimation theory (see \eg \citealp{polyanskiy2022information}).

\begin{example}[Conditional standard deviation]
When $\cX$ is a probability space and $\cY=\cZ=\R^d$, the measurable functions $f,h\in\bM(\cX,\R^d)$ are random variables and we recover that $\sigma_2(f|h) = \sqrt{\Exp{\|f - \Exp{f|h}\|^2}}$.
\end{example}

\begin{example}[Deviations and barycenters]
When $h$ is constant (\eg $\cY=\R$ and $h(x)=1$), then $\sigma_\nu(f|h) = \sigma_\nu(f|1) = \inf_{c\in\cZ} \|f - c\|_\nu$ can encode multiple notions of distance to the \emph{barycenter} of the values $\{f(x)\}_{x\in\cX}$, including the the median ($\nu=1$), mean ($\nu=2$) and Chebyshev center ($\nu=+\infty$) \citep{Amir1984}.
\end{example}

In what follows, we will mainly use this quantity for the semi-norms\footnote{The supremum in $\|f\|_{2,\cG}$ is a lattice supremum, \ie the smallest measurable function that is almost everywhere larger than all the considered functions, thus ensuring measurability of $\sup_{g\in\cG} \|f(g,z)\|^2$.}
$\|f\|_{\cG,2} = \sup_{g\in\cG} \sqrt{\Exp{\|f(g,z)\|^2}}$
and
$\|f\|_{2,\cG} = \sqrt{\Exp{\sup_{g\in\cG} \|f(g,z)\|^2}}$. 
Let $\sE_\cD: g \mapsto \ExpUnder{\xi\sim\cD}{g(\xi)}$ be the expectation over the distribution $\cD$. We denote as \emph{minimax estimation error} the quantity
\BEQ\label{eq:var_est}
\lb \,=\, \inf_{\varphi\in\bM(\cY,\cZ)} \|\sE_\cD - \varphi\circ \sO\|_{\cG,2}\,.
\EEQ
This quantity measures how well one can approximate the expectation of \emph{any} function in $\cG$ over the target distribution $\cD$ using the oracle $\sO$ as input. As we will see below, this quantity is tightly connected to the minimax excess risk.

\section{EXCESS RISK BOUNDS OF DATA-DEPENDENT ORACLES}
\label{sec:gen_bounds}

We now detail our upper and lower bounds on the minimax excess risk in various settings of interest.

\subsection{General data-dependent oracles}

We first provide a lower bound on the minimax excess risk that provides a link between this quantity and the minimax estimation error.
The proofs of all propositions are available in the supplementary material.

\begin{proposition}\label{prop:lb_GL}
For any distribution $\cD$, function class $\cG$ and data-dependent oracle $\sO$ verifying \Assumption{O}, we have
\BEQ\label{eq:lb_GL}
\mmratesc{\cG,\sO} \,\geq\, \frac{\lb^2}{2\mu}\,.
\EEQ
\end{proposition}

The proof of \Proposition{lb_GL} relies on simple well-chosen quadratic functions for which the observations of the gradient through the iterations of the optimization algorithm do not significantly change, and whose minimization requires to find a good estimator of the expectation over $\cD$ (\ie a solution to \Eq{var_est}).
Intuitively, \Proposition{lb_GL} shows that optimizing functions in $\cF_{\mbox{\tiny sc}}(\cG,\cD,\mu,L)$ is at least as difficult as estimating their gradient.
Moreover, the quantity $\lb$ can be lower bounded using any information theoretic lower bound on the variance of estimators. In particular, we will use a slight variation of Le Cam's two point method (see, e.g., Section 31.1 in \citealp{polyanskiy2022information}) adapted to our setting.

\begin{proposition}\label{prop:lecam}
For any distribution $\cD$, function class $\cG$ and data-dependent oracle $\sO$, we have
\BEQ
\lb^2 \geq \sup_{g,g'\in\cG} \frac{c_{g,g'}}{4}\,\|\sE_\cD(g) - \sE_\cD(g')\|^2\,,
\EEQ
where $c_{g,g'} = 1 - \dN{LC}(\sO(g,z), \sO(g',z))$ and $\dN{LC}(p,q)$ is Le Cam's distance (see \Appendix{lecam}).
\end{proposition}

In other words, if two functions $g,g'\in\cG$ are almost indistinguishable using observations (\ie Le Cam's distance between their respective distributions is small), then any estimator will necessarily return a similar value on both. As a consequence, if their expectations $\sE_\cD(g)$ and $\sE_\cD(g')$ are distant, then the estimator will have a large variance for at least one of the two functions. We will use this result in \Sec{applications} to derive lower bounds in several learning setups.

We now show that, if we replace $\|\cdot\|_{\cG,2}$ by the (always greater) norm $\|\cdot\|_{2,\cG}$, the lower bound in \Eq{lb_GL} can be achieved by a simple optimization algorithm.

\begin{proposition}\label{prop:ub_GL}
For any distribution $\cD$, function class $\cG$ and data-dependent oracle $\sO$, we have
\BEQ
\mmratesc{\cG,\sO} \,\leq\, \frac{\ub^2}{2\mu}\,.
\EEQ
\end{proposition}

The proof of \Proposition{ub_GL} relies on using the simple iterative algorithm $x_{t+1} = x_t - \frac{1}{L} \varphi(o_t)$ where $o_t = \sO(\nabla\ell_{x_t}, z)$ and $\varphi$ is a minimizer of \Eq{var_est}. Note that, if the oracle is $\sO(g,z)=\E_\cD[g(\xi)]$, this amounts to performing gradient descent. The variance of the gradient noise is then bounded by $\ub^2$ by taking the supremum over all functions $g\in\cG$ \emph{before} the expectation over $z$, thus avoiding issues related to the correlation between $x_t$ and $z$.
Of course, such a crude upper bound is often suboptimal, as $\ub$ allows for the function $g$ to be chosen adversarially for each random observation $z$, and thus does not take advantage of the independence between these two quantities. However, we now show that two additional assumptions lead to sharper upper bounds: 1) deterministic oracles and 2) \iid oracles (see \Sec{iid}).

\subsection{Exact risk with deterministic oracles}

Quite remarkably, the upper and lower bounds match in the case of deterministic oracles, thus providing an \emph{exact} relationship between minimax excess risk and minimax estimation error. 

\begin{corollary}\label{cor:deterministic}
If the observations are deterministic, \ie $\sO(g,z) = \tilde{\sO}(g)$ is independent of $z$, then
\BEQ
\mmratesc{\cG,\sO} \,=\, \frac{\sigma_\cG(\sE_\cD|\tilde{\sO})^2}{2\mu}\,,
\EEQ
where $\|\tilde{\sO}\|_\cG = \sup_{g\in\cG} \|\tilde{\sO}(g)\|$.
\end{corollary}

In \Sec{applications}, we will use this result to compute the minimax excess risk in one scenario: learning from fixed predetermined data-points (\eg a grid).

\subsection{Refined upper bounds with \iid oracles}
\label{sec:iid}

We now focus on the case where observations are of the form:
\BEQ\label{eq:iid_oracle}
\sO_n(g,z) = \left( \sO(g, z^{(1)}), \dots, \sO(g, z^{(n)}) \right)\,,
\EEQ
where $z = (z^{(1)}, \dots, z^{(n)})$ and the $z^{(i)}$ are \iid random variables. 
For example, if the random variables $z^{(i)}$ are sampled from $\cD$ and $\sO(g,z) = g(z)$, this amounts to classical supervised learning with $n$ samples.
We first provide an upper bound on the minimax excess risk using a simple mini-batch algorithm with warmup.

\begin{proposition}\label{prop:iid}
Let $a > 0$ and $\kappa=L/\mu$. For any distribution $\cD$, function class $\cG$ and \iid oracle $\sO_n$, 
\BEQe
\mmratesc{\cG,\sO_n} \,\leq\, \frac{\sigma_{\cG,2}\left( \sE_\cD | \sO_{\tilde{n}} \right)^2}{2\mu} + \frac{\widetilde{\Delta}}{n^a}\,,
\EEQe
where $\tilde{n} = \left\lfloor\frac{n-1}{1+a\kappa\log n}\right\rfloor$, and $\widetilde{\Delta} = \frac{\sigma_{2,\cG}\left( \sE_\cD | \sO_1 \right)^2}{2\mu}$.
\end{proposition}

\begin{algorithm}[t]
\caption{Minibatch GD with warmup}\label{alg:mini-batch}
\begin{algorithmic}
\REQUIRE iterations $T$, sizes $(n_t)_{t<T}$, functions $(\varphi_k)_{k\in\N^*}$
\ENSURE current iterate $x$
\STATE $m\gets 0, x \gets 0, o\gets\sO(\nabla\ell_x, z^{(1)})$
\STATE $T_{\mbox{\tiny wu}}\gets\kappa\ln\left(\frac{\|\varphi_1(o)\|^2 + \|\sE_\cD-\varphi_1\circ\sO\|_{2,\cG}^2}{\varepsilon\mu}\right)$
\FOR{$t\in\set{0,T_{\mbox{\tiny wu}}-1}$}
\STATE $o\gets\sO(\nabla\ell_x, z^{(1)})$
\STATE $x \gets x - \frac{1}{L} \varphi_1(o)$
\ENDFOR
\FOR{$t\in\set{0,T-1}$}
\STATE $o_i\gets\sO(\nabla\ell_x, z^{(m+i)})$ for $i\in\set{1,n_t}$
\STATE $x \gets x - \frac{1}{L} \varphi_{n_t}(o_1,\dots,o_{n_t})$
\STATE $m\gets m + n_t$
\ENDFOR
\end{algorithmic}
\end{algorithm}

Similarly to \Proposition{ub_GL}, the proof of \Proposition{iid} relies on the use of an iterative algorithm akin to a gradient descent variant, here mini-batch gradient descent with a warm-up phase. The algorithm, described in \Alg{mini-batch}, requires functions $\varphi_k$ for $k\in\N^*$ minimizing \Eq{var_est} for the oracle $\sO_k$, and mini-batch sizes $n_t=\tilde{n}$. 
Note that, apart from the first sample used during the warmup phase, samples are only used once. This may seem suboptimal, as stability theory shows that one can often reuse samples without a significant cost. However, note that the stability of \Alg{mini-batch} depends on the regularity of the functions $\varphi_k$, which is not controlled in general. Moreover, \Proposition{iid} shows that, even with a simple mini-batch scheme without replacement, one can already obtain matching upper and lower bounds up to logarithmic factors, as $\tilde{n}=\Omega(n/\kappa\log{n})$.
In particular, for oracles of the form of \Eq{iid_oracle} for which $\sigma_{2,\cG}\left( \sE_\cD | \sO_1 \right) < +\infty$ and $\sigma_{\cG,2}\left( \sE_\cD | \sO_n \right)=\Theta(n^{-b})$, \Proposition{iid} implies that $\mmratesc{\cG,\sO_n} = \widetilde{\Theta}(n^{-2b})$, 
where $\widetilde\Theta$ hides logarithmic factors (using \Proposition{iid} with $a\geq 2b$).
Finally, the logarithmic factor can be removed when the minimax estimation error is bounded by a quantity of the form $a+b/n$ (see \Sec{applications} for examples of such a bound).

\begin{proposition}\label{prop:iid_increasing_batch}
Let $a,b > 0$, $n\geq 3$, and assume that $\sigma_{\cG,2}\left( \sE_\cD | \sO_n \right)^2 \leq a + b/n$. Then, for any function class $\cG$, distribution $\cD$ and \iid oracle $\sO_n$, we have
\BEQe
\mmratesc{\cG,\sO_n} \,\leq\, \frac{a}{2\mu} + \frac{6\kappa b}{\mu n} + \widetilde{\Delta} e^{-\frac{n}{6\kappa}}\,,
\EEQe
where $\widetilde{\Delta} = \frac{\sigma_{2,\cG}\left( \sE_\cD | \sO_1 \right)^2}{2\mu}$ and $\kappa=L/\mu$.
\end{proposition}

This result is obtained with exponentially increasing mini-batch sizes $n_t = \left\lceil n (1-c)c^{T-t-1}/2 \right\rceil$ where $c = \sqrt{1-\kappa^{-1}}$, and $T=\lfloor n/2 \rfloor$. This allows to have more precision at the end of the optimization, when the error is low and greater precision is required.

\section{APPLICATIONS}\label{sec:applications}

\begin{table*}[t]
\caption{Our upper and lower bounds on the minimax excess risk $\mmratesc{\cG,\sO}$ in several learning scenarios, and up to multiplicative universal constants whose values are available in the appendix.} \label{tab:results}
\begin{center}
\begin{tabular}{l|cc}
\textbf{SCENARIO} & \textbf{LOWER BOUND} & \textbf{UPPER BOUND}\\
\hline \\
Supervised  ($\cGN{Bnd}$) & $\frac{B^2}{\mu n}$ & $\frac{\kappa B^2}{\mu n}$\\
Transfer  ($\cGN{Bnd}$) & $\frac{B^2}{\mu}\left( \dN{TV}(\cD,\cD')^2 + \frac{1}{n}\right)$ & $\frac{B^2}{\mu}\left( \dN{TV}(\cD,\cD')^2 + \frac{\kappa}{n}\right)$\\
Federated  ($\cGN{Bnd}$) & unknown 
& $\inf_{q\in\R^m}\frac{B^2}{\mu}\left( \dN{TV}(\cD,\cD_q)^2 + \sum_{i=1}^m \frac{\kappa q_i^2}{n_i}\right)$\\
Robust  ($\cGN{Bnd}$) & $\frac{B^2}{\mu} \left(\eta^2 + \frac{1}{n}\right)$ & $\frac{B^2}{\mu} \left(\eta^2 + \frac{\kappa}{n}\right)$ \\
Robust  ($\cGN{Lip}$) & unknown & $\frac{B^2}{\mu} \var(\xi) \left(\eta + \frac{\kappa}{n}\right)$ \\
Fixed data ($\cGN{Bnd}$) & $\frac{B^2}{\mu}\left(1 - \ProbUnder{\cD}{\{\xi_i'\}_{i\set{1,n}}}\right)^2$  & $\frac{B^2}{\mu}\left(1 - \ProbUnder{\cD}{\{\xi_i'\}_{i\set{1,n}}}\right)^2$ \\
Fixed data ($\cGN{Lip}$) & $\frac{B^2}{\mu}\Exp{\min_i\|\xi - \xi_i\|}^2$  & $\frac{B^2}{\mu}\Exp{\min_i\|\xi - \xi_i\|}^2$ \\
\end{tabular}
\end{center}
\end{table*}

In this section, we specify the general upper and lower bounds obtained in the previous sections to several statistical learning scenarios. The list of results obtained for the Lipschitz and bounded variation function classes $\cGN{Lip}$ and $\cGN{Bnd}$ are reported in \Tab{results}.

\subsection{Supervised learning}
We now consider the typical supervised learning setup, in which the training samples are drawn \iid according to the target distribution, \ie
\BEQe
\sON{SL}(g) = (g(\xi_1),\dots,g(\xi_n))\,,
\EEQe
where $\xi_i\sim\cD$ are \iid random variables.
A classical approach when using this oracle consists in minimizing the empirical loss, by computing $\hat x_n\in\argmin_{x\in\R^d} \frac{1}{n}\sum_i\ell(x,\xi_i)$ using for instance a gradient descent algorithm.

Applying \Proposition{iid} to the oracle $\sON{SL}$ gives an upper bound involving $\sigma_{\cG,2}(\sE_\cD | \sO^{\mbox{\rm\tiny SL}}_{\tilde{n}})^2/2\mu$ on the minimax excess risk (for $\tilde{n}$ specified in \Proposition{iid}, of order $\widetilde\cO(n/\kappa)$). This quantity is however hard to compute in most cases, as $\sigma_{\cG,2}(\sE_\cD|\sO^{\mbox{\rm\tiny SL}}_{\tilde{n}})$ is a minimum over all measurable functions $\varphi\in\bM(\cO,\R^d)$ (see \Eq{var_est}). However, a simple upper bound can be obtained by using the average over the samples $\varphi(g(\xi_1),\dots,g(\xi_n)) = \frac{1}{n}\sum_i g(\xi_i)$. The use of such a function leads to the usual (mini-batch) gradient descent algorithm on the empirical risk, and to the following proposition, which shows that $\sigma_{2,\cG}(\sE_\cD|\sON{SL})$ can be upper bounded by the form specified in \Proposition{iid_increasing_batch}.
\begin{proposition}\label{prop:upper_SL}
For $n\geq 1$, we have
\BEQ
\sigma_{\cG,2}(\sE_\cD|\sON{SL})^2 \quad\leq\quad \frac{\SV{\cD}{\cG}}{n}\,,
\EEQ
where $\SV{\cD}{\cG} = \sup_{g\in\cG} \var(g(\xi))$ and $\xi\sim\cD$.
\end{proposition}

The quantity $\SV{\cD}{\cG} = \sup_{g\in\cG} \var(g(\xi_1))$ controls the variation of the gradient of the loss over the data distribution, and is easy to compute for simple function classes defined in \Sec{stat_learn}: 
1) Affine functions: $\SV{\cD}{\cGN{Aff}}\leq B^2\var(\xi)$ (with equality if $\Xi=\R^D$ and $D\leq d$), 
2) Lipschitz functions: $\SV{\cD}{\cGN{Lip}}\leq B^2\var(\xi)$ (with equality if $\Xi=\R^D$ and $D\leq d$), 
and 3) Bounded variation: $\SV{\cD}{\cGN{Bnd}}\leq B^2$ (with equality if $\exists A\subset \Xi$ mesurable \st $\ProbUnder{\cD}{A}=\frac{1}{2}$).
 
In particular, our results provide new excess risk bounds for the set of smooth and strongly convex (or PL, see \Appendix{PL}) functions whose gradient is Lipschitz \wrt $x$ and \wrt input data (\ie $\cG=\cGN{Lip}$), by applying the bound $\SV{\cD}{\cGN{Lip}}\leq B^2\var(\xi)$ to \Proposition{upper_SL} and \Proposition{iid_increasing_batch}. 

We now explicit our bounds on a classical setting for which we can compare our results: the bounded variation function class $\cGN{Bnd}$, in which the gradients are contained in a ball, and includes regularized Lipschitz losses (see \Example{reglip}).

\begin{proposition}\label{prop:bounded}
Assume that $\forall c\in[0,1]$, $\exists A\subset\Xi$ measurable \st $\ProbUnder{\cD}{A} = c$. Then, for $n\geq 3$, we have
\BEQ
\label{eq:boundedSL}
\frac{B^2}{8\mu n} \,\leq\, \mmratesc{\cGN{Bnd},\sON{SL}} \,\leq\, \frac{11\kappa B^2}{\mu n}\,.
\EEQ
\end{proposition}

The upper and lower bounds in \Proposition{bounded} match up to a multiplicative factor proportional to $\kappa$, thus providing a relatively tight approximation of the minimax excess risk in this setting.
Also, note that the assumption \wrt the measure $\cD$ in the previous proposition is only necessary for the lower bound to hold, and is automatically verified for continuous distributions. 

\paragraph{Comparison with the literature.} As mentioned above, the bounded variation setup includes the widely studied case of $B$-Lipschitz continuous loss functions. This setting is handled by \citet{sridharan2008fast,bartlett2005local} and \citet[Chapter 4.5.5]{bach2021learning} for regularized risk minimization, and under additional structural assumptions on $\ell$. It is also treated by the stability community, such as in \citet{bousquet2002stability} for regularized objectives, or more generally for $\mu$-strongly convex functions in \citet{hardt2016train}. In all these works, the authors study the excess risk $\cL(\hat x_n)-\inf_{x\in\R^d}\cL(x)$, for $\hat x_n$ function of $\sON{SL}$ and corresponding to the empirical minimizer. They provide bounds of the form $O\left(\frac{B^2}{\mu n}\right)$, which, up to a multiplicative factor proportional to $\kappa$, is the same as ours in Equation \eqref{eq:boundedSL}. Such multiplicative factor could be avoided by considering an optimization algorithm that uses the training samples several times, therefore minimizing effectively the empirical risk (see the discussion below \Proposition{iid}). Such analysis could however necessitate the introduction of a notion of stability w.r.t. the optimization algorithm and is currently kept for future work. Finally, note that our lower bound demonstrates the optimality of the state-of-the-art upper bound in $O\left(\frac{B^2}{\mu n}\right)$, a result which was not provided by the aforementioned works.

\subsection{Transfer learning}

We now turn to the Transfer Learning (TL) oracle, defined as
\BEQe
\sON{TL}(g) = (g(\xi_1'), \dots, g(\xi_n'))\,,
\EEQe 
where $\xi_i'\sim\cD'$ are \iid random variables and $\cD'\neq\cD$. 
This oracle typically encompasses applications in transfer learning such as domain adaptation \citep{bendavid2006domain}, where a task is learnt on a training dataset $\cD'$ that differs from the test distribution $\cD$.

\begin{proposition}\label{prop:TLrad}
We have
\BEQ\label{eq:TL}
\sigma_{\cG,2}(\sE_\cD|\sON{TL})^2 \,\leq\, d_\cG(\cD,\cD')^2 + \frac{\SV{\cD'}{\cG}}{n}\,,
\EEQ
where $d_\cG(\cD,\cD') = \sup_{g\in\cG}\|\sE_\cD(g) - \sE_{\cD'}(g)\|$ is an integral probability metric (IPM).
\end{proposition}

The Integral Probability Metrics $d_\cG(\cD,\cD')$ for the function classes $\cGN{Bnd}$, $\cGN{Lip}$, and $\cGN{Aff}$ defined in \Sec{stat_learn} give respectively, the total variation distance, the Wasserstein distance, and the distance between expectations $\|\ExpUnder{\cD}{\xi} - \ExpUnder{\cD'}{\xi'}\|$ in $\R^D$.
From \Eq{TL}, note also that $\sigma_{\cG,2}(\sE_\cD|\sON{TL})^2$ is of the form specified in \Proposition{iid_increasing_batch}. Finally, the case $\cGN{Bnd}$ also provides a lower bound on the minimax estimation error.
\begin{proposition}\label{prop:TLBnd}
Assume that $\cD \ll \cD'$ and $\forall c\in[0,1], \exists q\in\R$ \st $\ProbUnder{\cD'}{\frac{d\cD}{d\cD'}(\xi') \geq q} = c$.
Then, the minimax estimation error $\sigma_{\cGN{Bnd},2}(\sE_{\cD}|\sON{TL})^2$ is 
\BEQ\label{eq:TLBnd}
    \Theta\left(B^2 \left(\dN{TV}\big(\cD,\cD'\big)^2 + \frac{1}{n}\right)\right)\,,
\EEQ
\end{proposition}
where $\dN{TV}$ is the Total Variation distance (Appendix~\ref{appendix:lecam}).
Using \Proposition{iid_increasing_batch}, we obtain upper and lower bounds on the minimax excess risk that are within a multiplicative factor proportional to $\kappa$ (see \Tab{results}).

\paragraph{Comparison with the literature.} Our bound can be put into perspective with the generalization bounds derived by the Domain Adaptation (DA) community. For instance, 
\citet{bendavid2006domain,blitzer2007learning} provide excess risk bounds for DA of the form $O\big(\sqrt{d_{\rm VC}/n} + d_\cH(\cD,\cD')\big)$ for algorithms learning from $n$ samples drawn from a training data distribution $\cD'$, on a test distribution $\cD$, for a hypothesis class $\cH$ of VC dimension $d_{\rm VC}$.
From \Proposition{TLrad}, we can prove (see \Tab{results}), that in the bounded variation setting, the minimax excess risk can be upper bounded by $\mmratesc{\cGN{Bnd},\sON{TL}} = \cO(\frac{B^2}{\mu}\left( \dN{TV}(\cD,\cD')^2 + \frac{\kappa}{n}\right))$. 
At first sight we could conclude that our bound is always better than the first one since it exhibits a fast rate with respect to $n$ and since $d_{\rm VC}$ is significantly larger (it can be infinite) than the constants of our bound. However, these two bounds cannot directly be compared as the setups are not perfectly matching. In particular, we are able to obtain a fast rate component in $\cO(1/n)$ thanks to the strong convexity assumption, a setup which, to the best of our knowledge, was not explicitly considered in previous analyses (see \cite{redko2020survey} for a survey on the theoretical guarantees of DA). A more detailed discussion on the difference between our distance on the gradient $d_\cG(\cD,\cD')$ and that of prior works on the function value $d_\cH(\cD,\cD')$ is available in \Appendix{dH_vs_dG}.

\subsection{Federated learning}

We now consider a setting in which $m$ local agents are willing to collaborate in order to minimize their shared (or sometimes personal) excess risk, a setup known as (Personalized) Federated Learning (FL) \citep{advances_FL}. Let $(\cD_i)_{i\in\set{1,m}}$ be a set of \emph{local} distributions, 
and let $\sON{FL}(g,z) = (g(\xi^i_j))_{i\in\set{1,m},j\in\set{1,n_i}}$ for $\xi^i_j\sim\cD_i$ \iid random variables ($n_i$ samples from agent $i$).
We have the following minimax excess risk upper bounds that extend previous results \citep{even2022on,ding2022collaborative}, for the (P)FL oracle and objective (where the objective distribution $\cD$ is $\cD=\sum_i p_i\cD_i$ for FL and $\cD=\cD_1$ for PFL).
\begin{proposition}\label{prop:federated}
We have
\BEQe
    \sigma_{\cG,2}(\sE_{\cD}|\sON{FL})^2 \leq \inf_{q\in\R^m} d_\cG\big(\cD,\cD_q\big)^2 + \sum_{i=1}^m \frac{q_i^2\SV{\cD_i}{\cG}}{n_i}\,,
\EEQe
where $\cD_q = \sum_{i=1}^m q_i\cD_i$.
\end{proposition}
Intuitively, this bound allows to trade bias on the target distribution (first term) with variance of the local gradient (second term).
Note that $\cD_q$ may not be a probability measure, as the weights $q_i$ are not necessarily positive and summing to $1$.

\subsection{Robust learning}
We now consider a setting in which a fraction $\eta$ of the data points may be arbitrarily corrupted \citep{https://doi.org/10.48550/arxiv.1803.03241}. To simplify the analysis, we will assume that these outliers are drawn according to an unknown (potentially very bad) distribution $\cD_o$. The oracle is thus defined as
\BEQ
\sON{RL}(g,z) = (g(\xi_1'),\dots,g(\xi_n'))\,,
\EEQ
where $\xi_i'\sim (1-\eta)\cD + \eta\cD_o$ are \iid random variables.
Note that this setting can be considered as a particular case of transfer learning. However, our objective is to obtain bounds that do not depend on the outlier distribution $\cD_o$, and we thus focus on distant outliers such that $\dN{TV}\big(\cD,\cD_o\big) = 1$ (highest possible value for the total variation).
First, the bounded variation setting is here very simple, as the ouliers cannot perturb the estimation to a large degree. In such a case, applying \Proposition{TLrad} where $\cD'$ is the corrupted training dataset gives the upper bound:
\BEQ
\sigma_{\cGN{Bnd},2}(\sE_\cD|\sON{RL})^2 = \Theta\left(B^2 \left(\eta^2 + \frac{1}{n}\right)\right)\,.
\EEQ
In the more challenging (and realistic) case of Lipschitz gradients w.r.t. the data points, the outliers can reach very large gradient values and thus completely break the average. To avoid this issue, we can use the robust mean estimation algorithm in \citet[Algorithm 1]{SteinhardtCV18} as estimator of the gradient of the population risk. This gives the following upper bound: 
\BEQ
\sigma_{\cGN{Lip},2}(\sE_\cD|\sON{RL})^2 \leq c B^2 \var(\xi) \left(\eta + \frac{1}{n}\right)\,,
\EEQ
where $c$ is a universal constant, $\eta\leq 1/4$ is the fraction of outliers, $n$ is the total number of samples and $\var(\xi)$ is the variance of the true data distribution (without the outliers).

\subsection{Learning from fixed data-points}
In supervised learning, the \iid assumption on the training dataset is key to obtain fast convergence \wrt the number of samples. However, training data-points are sometimes imposed and predetermined, for example following a pattern such as once every day or year for temporal data, or on a 2d grid for geophysical data (\eg weather forecasts).
In such a case, the minimax excess risk will depend on the distance between this training data and the target distribution.
We thus consider the oracle defined as
\BEQ
\sON{FD}(g,z) = (g(\xi_1'),\dots,g(\xi_n'))\,,
\EEQ
where $(\xi_1',\dots,\xi_n')$ are fixed prior to the optimization.
As the oracle is deterministic, \Corollary{deterministic} allows to obtain the exact value of the minimax excess risk.

\begin{proposition}\label{prop:fixed_learn}
We have
\BEQ
\mmratesc{\cGN{Bnd},\sON{FD}} = \frac{2B^2}{\mu}\left(1 - \ProbUnder{\cD}{\{\xi_i'\}_{i\set{1,n}}}\right)^2
\EEQ
and
\BEQ
\mmratesc{\cGN{Lip},\sON{FD}} = \frac{B^2}{2\mu}\Exp{\min_i\|\xi - \xi_i'\|}^2\,.
\EEQ
\end{proposition}
As $\cGN{Bnd}$ does not assume any local regularity \wrt data, knowing the value of the gradient on the data points does not provide any information on the gradient on the rest of the distribution. However, the Lipschitz assumption allows for smaller minimax excess risk that tends to $0$ as the number of samples $n$ tends to $+\infty$.

\section{CONCLUSION}

In this paper, we introduced a novel unified framework for the minimax excess risk control of a large panel of statistical learning problems. We focused on first-order optimization methods with data-dependent gradient oracles and showed, thanks to the new notion of \emph{best approximation error}, that what matters is the ability of the given gradient oracle to approximate the true gradient of the population risk. Thanks to our general framework that encompasses numerous applications, we showed that this notion leads to sharp minimax excess risk bounds in most considered cases. 

Our work focuses on specific regularity assumptions and applications due to lack of space and for clarity of exposition; we believe that our promising results and framework extend to other classical regularity assumption sets, and to other applications mentioned in our paper, which we leave for future work.

\subsubsection*{Acknowledgements}
This work was supported by the French government managed by the Agence Nationale de la Recherche (ANR) through France 2030 program with the reference ANR-23-PEIA-005 (REDEEM project). It was also funded in part by the Groupe La Poste, sponsor of the Inria Foundation, in the framework of the FedMalin Inria Challenge. Laurent Massoulié was supported by the French government under management of Agence Nationale de la Recherche as part of the “Investissements d’avenir” program, reference ANR19-P3IA-0001 (PRAIRIE 3IA Institute).

\bibliographystyle{apalike}
\bibliography{bibliography}

\newpage
\section*{Checklist}

 \begin{enumerate}

 \item For all models and algorithms presented, check if you include:
 \begin{enumerate}
   \item A clear description of the mathematical setting, assumptions, algorithm, and/or model. [Yes]
   \item An analysis of the properties and complexity (time, space, sample size) of any algorithm. [Yes]
   \item (Optional) Anonymized source code, with specification of all dependencies, including external libraries. [Not Applicable]
 \end{enumerate}

 \item For any theoretical claim, check if you include:
 \begin{enumerate}
   \item Statements of the full set of assumptions of all theoretical results. [Yes]
   \item Complete proofs of all theoretical results. [Yes]
   \item Clear explanations of any assumptions. [Yes]     
 \end{enumerate}

 \item For all figures and tables that present empirical results, check if you include:
 \begin{enumerate}
   \item The code, data, and instructions needed to reproduce the main experimental results (either in the supplemental material or as a URL). [Not Applicable]
   \item All the training details (e.g., data splits, hyperparameters, how they were chosen). [Not Applicable]
         \item A clear definition of the specific measure or statistics and error bars (e.g., with respect to the random seed after running experiments multiple times). [Not Applicable]
         \item A description of the computing infrastructure used. (e.g., type of GPUs, internal cluster, or cloud provider). [Not Applicable]
 \end{enumerate}

 \item If you are using existing assets (e.g., code, data, models) or curating/releasing new assets, check if you include:
 \begin{enumerate}
   \item Citations of the creator If your work uses existing assets. [Not Applicable]
   \item The license information of the assets, if applicable. [Not Applicable]
   \item New assets either in the supplemental material or as a URL, if applicable. [Not Applicable]
   \item Information about consent from data providers/curators. [Not Applicable]
   \item Discussion of sensible content if applicable, e.g., personally identifiable information or offensive content. [Not Applicable]
 \end{enumerate}

 \item If you used crowdsourcing or conducted research with human subjects, check if you include:
 \begin{enumerate}
   \item The full text of instructions given to participants and screenshots. [Not Applicable]
   \item Descriptions of potential participant risks, with links to Institutional Review Board (IRB) approvals if applicable. [Not Applicable]
   \item The estimated hourly wage paid to participants and the total amount spent on participant compensation. [Not Applicable]
 \end{enumerate}

 \end{enumerate}

\newpage
\onecolumn
\appendix

\section{Link with the Polyak-\Loja{} condition}\label{appendix:PL}

Recall the definition of the Polyak-\Loja{} condition, which is weaker than strong convexity and can be satisfied by non-convex functions. 

\begin{definition}(Polyak-\Loja{})
    Let $f:\R^d\rightarrow\R$ be differentiable, and $\mu>0$. We say that $f$ is $\mu$-Polyak-\Loja{} ($\mu$-PL for short) if it is bounded from below, and if for all $x\in\R^d$
    \BEQ
    f(x) - \inf f \leq \frac{1}{2\mu}\|\nabla f(x)\|^2\; .
    \EEQ
\end{definition}

Importantly, $\mu$-strongly convex functions are also $\mu$-PL and we therefore have $\cF_{\mbox{\tiny sc}}(\cG,\cD, \mu, L) \subset \cF_{\mbox{\tiny pl}}(\cG,\cD, \mu, L)$ (abbreviated to $\cF_{\mbox{\tiny pl}}$ below), the set of $\mu$-PL and $L$-smooth objective functions $\cL(x) = \ExpUnder{\xi\sim\cD}{\ell(x,\xi)}$ such that $\forall x\in\R^d$, $\nabla\ell_x \in \cG$. As an immediat consequence, we also have the relation $\mmratesc{\cG,\sO} \leq \mmratepl{\cG,\sO}$ and all the lower bounds presented for $\mmratesc{\cG,\sO}$ in the main paper are also immediately valid for $\mmratepl{\cG,\sO}$. 

To make all our results valid for both set of functions $\cF_{\mbox{\tiny sc}}$ and $\cF_{\mbox{\tiny pl}}$, all the upper bounds derived in the paper are actually proved, in the following sections, for $\mu$-PL objective functions. 

\section{Le Cam's distance between probability distributions}\label{appendix:lecam}

First, we recall the definition of two standard divergences between probability distributions. Let $P,Q$ be two probability distributions such that $dP(dx) = p(x)d\mu(x)$ and $dQ(x) = q(x)d\mu(x)$ for some common dominating measure $\mu$.
\BIT
\item \textbf{$f$-divergences:} Let $f:\R_+\mapsto\R\cup\{+\infty\}$ be a convex function such that $f(1) = 0$ and $\lim_{t\to 0^+} f(t) = f(0)$. The $f$-divergence between $P$ and $Q$ is defined as
\BEQ
D_f(P,Q) = \int f\left( \frac{p(x)}{q(x)} \right) q(x) d\mu(x)\,.
\EEQ
\item \textbf{Total variation:} The total variation distance is defined as
\BEQ
\dN{TV}(P,Q) = \frac{1}{2}\int |p(x) - q(x)| d\mu(x)\,.
\EEQ
\item \textbf{Kullback-Leibler:} The Kullback-Leibler divergence is defined as
\BEQ
\dN{KL}(P,Q) = \int \ln\left( \frac{p(x)}{q(x)}\right)p(x) d\mu(x)\,.
\EEQ
\EIT

Note that both $\dN{TV}$ and $\dN{KL}$ are $f$-divergences with, respectively, $f(t) = |t-1|/2$ and $f(t) = t\ln(t)$.  Below, we recall a useful property of $f$-divergences that is going to be used later.

\begin{property} \label{pty:divergence}$ D_f = D_h$ if and only if $f(t) = h(t) + c(t-1)$ for some constant $c\in\R$.
\end{property}

We now provide a definition for Le Cam's distance.

\begin{definition}
Let $P,Q$ be two probability distributions such that $dP(x) = p(x)\mu(x)$ and $dQ(x) = q(x)d\mu(x)$ for some common dominating measure $\mu$. We denote as \emph{Le Cam's distance} between $P$ and $Q$ the quantity
\BEQ
\dN{LC}(P,Q) = \frac{1}{2}\int \frac{(p(x) - q(x))^2}{p(x) + q(x)} d\mu(x)\,.
\EEQ
\end{definition}

Another definition for $\dN{LC}$ is the $f$-divergence obtained with the (convex) function $f(t) = \frac{(1-t)^2}{2(1+t)}$, or equivalently thanks to Property \ref{pty:divergence}, with the function $h(t)=f(t)-\frac{1}{2}(t-1) = \frac{1-t}{1+t}$.

By definition, $\dN{LC}$ is symmetric, and we have the following relationship between Le Cam's distance and other standard $f$-divergences, thanks again to Property \ref{pty:divergence}.

\begin{lemma}
For any $P,Q$, we have
\BEQ
\dN{LC}(P,Q)\leq \dN{TV}(P,Q) \quad\mbox{ and }\quad \dN{LC}(P,Q)\leq \dN{KL}(P,Q)\,.
\EEQ
\end{lemma}
\begin{proof}
A simple functional analysis gives $\frac{1-t}{1+t} + \frac{t-1}{2} \leq \frac{1}{2}|t-1|$  for $t\geq 0$, thus directly implying the first inequality. For the second, we have $\frac{1-t}{1+t} \leq -\ln(t) + \frac{t-1}{2}$ for $t\geq 0$ and, as the $f$-divergence with $f(t) = -\ln(t)$ corresponds to the reverse KL and $\dN{LC}$ is symmetric, we have $\dN{LC}(P,Q) = \dN{LC}(Q,P) \leq \dN{KL}(P,Q)$.
\end{proof}

The link with the Kullback-Leibler divergence will be useful to derive proofs in the \iid oracle regime.
We now provide a proof of Le Cam's two point method adapted to our setting.

\begin{proof}[\textbf{Proof of \Proposition{lecam}}]
Let $g_1,g_2\in\cG$ be two functions in the functions class. Then,
\BEQ
\BA{lll}
\lb^2 &=& \inf_{\varphi\in\bM(\cY,\cZ)} \sup_{g\in\cG} \ExpUnder{z}{\|\sE_\cD(g) - \varphi\circ \sO(g,z)\|^2}\\
&\geq& \inf_{\varphi\in\bM(\cY,\cZ)} \sup_{g\in\{g_1,g_2\}} \ExpUnder{z}{\|\sE_\cD(g) - \varphi\circ \sO(g,z)\|^2}\\
&\geq& \inf_{\varphi\in\bM(\cY,\cZ)} \ExpUnder{G,z}{\|\sE_\cD(G) - \varphi\circ \sO(G,z)\|^2}\,,\\
\EA
\EEQ
where $G = B g_1 + (1-B)g_2$ and $B\sim\cB(1/2)$ is a Bernoulli random variable of parameter $1/2$. The infimum over measurable functions $\varphi$ is now attained for the conditional expectation $\Exp{\sE_\cD(G)|\sO(G,z)}$, and a simple calculation gives
\BEQ
\Exp{\sE_\cD(G)|\sO(G,z)} = \sE_\cD(g_1)\frac{p_1(O(G,z))}{p_1(O(G,z)) + p_2(O(G,z))} + \sE_\cD(g_2) \frac{p_2(O(G,z))}{p_1(O(G,z)) + p_2(O(G,z))}\,,
\EEQ
where $p_1,p_2$ are the Radon-Nykodym densities of, respectively, $\sO(g_1,z)$ and $\sO(g_2,z)$ \wrt to a common dominating measure $\mu$. Combining the two previous equaitons, we get
\BEQ
\BA{lll}
\lb^2 &\geq& \ExpUnder{G,z}{\|\sE_\cD(G) - \frac{\sE_\cD(g_1)\, p_1(O(G,z)) + \sE_\cD(g_2)\, p_2(O(G,z))}{p_1(O(G,z)) + p_2(O(G,z))}\|^2}\\
&=& \frac{1}{2}\ExpUnder{z}{\|(\sE_\cD(g_1) - \sE_\cD(g_2))\frac{p_2(O(g_1,z))}{p_1(O(g_1,z)) + p_2(O(g_1,z))}\|^2 + \|(\sE_\cD(g_1) - \sE_\cD(g_2))\frac{p_1(O(g_2,z))}{p_1(O(g_2,z)) + p_2(O(g_2,z))}\|^2}\\
&=& \frac{\|\sE_\cD(g_1) - \sE_\cD(g_2)\|^2}{2} \left(\int \frac{p_2(o)^2}{(p_1(o) + p_2(o))^2}p_1(o) d\mu(o) + \int \frac{p_1(o)^2}{(p_1(o) + p_2(o))^2}p_2(o) d\mu(o)\right)\\
&=& \frac{\|\sE_\cD(g_1) - \sE_\cD(g_2)\|^2}{2} \int \frac{p_1(o)p_2(o)^2 + p_2(o)p_1(o)^2}{(p_1(o) + p_2(o))^2} d\mu(o)\\
&=& \frac{\|\sE_\cD(g_1) - \sE_\cD(g_2)\|^2}{2} \int \frac{p_1(o)p_2(o)}{p_1(o) + p_2(o)} d\mu(o)\\
&=& \frac{\|\sE_\cD(g_1) - \sE_\cD(g_2)\|^2}{4} \left(1 - \frac{1}{2}\int \frac{(p_1(o)-p_2(o))^2}{p_1(o) + p_2(o)} d\mu(o)\right)\\
&=& \frac{\|\sE_\cD(g_1) - \sE_\cD(g_2)\|^2}{4} \Big(1 - \dN{LC}(O(g_1,z),O(g_2,z))\Big)
\EA
\EEQ
\end{proof}

\section{Comparison between distances on gradients and function values}\label{appendix:dH_vs_dG}
We now discuss the differences between $d_\mathcal{H}(\mathcal{D},\mathcal{D}')$ where $\mathcal{H} \supset \{\xi\mapsto \ell(x,\xi)~:~x\in\mathbb{R}^d\}$ contains the values of the loss, and $d_\mathcal{G}(\mathcal{D},\mathcal{D}')$, where $\mathcal{G} \supset \{\xi\mapsto \nabla_x \ell(x,\xi)~:~x\in\mathbb{R}^d\}$ contains gradients of the loss.
First, note that discrepancies in loss value are usually used to control the generalisation error on the iterates, as $|\mathbb{E}[\ell(x_{\sA[\sO]}^{(t)},\xi)] - \mathbb{E}[\ell(x_{\sA[\sO]}^{(t)},\xi')]| \leq d_\mathcal{H}(\mathcal{D},\mathcal{D}')$ as long as $\xi\mapsto\ell(x_{\sA[\sO]}^{(t)},\xi)\in\mathcal{H}$. Unfortunately, these discrepancies are infinite in our setting, as we now show:
If $\mathcal{H}_1 = \{\xi\mapsto \ell(x,\xi)~:~x\in\mathbb{R}^d \mbox{ and } \mathbb{E}[\ell(\cdot,\xi)]\in\mathcal{F}_{sc}(\mathcal{G},\mathcal{D},\mu,L)\}$ is the set of strongly-convex and smooth loss functions considered in this paper, then choosing $\ell^g(x,\xi) = \frac{\mu}{2}\|x\|^2 + \langle g(\xi),x\rangle$ gives
\BEQ
d_{\mathcal{H}_1}(\mathcal{D},\mathcal{D}') \geq \sup_{x\in\mathbb{R}^d}|\mathbb{E}[\ell^g(x,\xi) - \ell^g(x,\xi')]| = \sup_{x\in\mathbb{R}^d} |\langle\mathbb{E}[g(\xi)] - \mathbb{E}[g(\xi')],x\rangle| = +\infty\,,
\EEQ
as soon as $\mathbb{E}[g(\xi)] \neq \mathbb{E}[g(\xi')]$ (i.e. $d_{\mathcal{G}}(\mathcal{D},\mathcal{D}') > 0$).
The discrepancy in function value is thus unsuited to the strongly convex and smooth setting without additional assumptions on the domain of $x$ or boundedness of the loss.
However, one could argue that the difference in loss is only necessary on the algorithm's output $x_{\sA[\sO]}$ (or equivalently the algorithm's iterates $x_{\sA[\sO]}^{(t)}$) instead of the whole space. Unfortunately, this quantity is also infinite, as, $\forall g\in\mathcal{G}$ and $\forall c\in\mathbb{R}^d$, the function $\ell^{g,c}(x,\xi) = \frac{\mu}{2}\|x\|^2 + \langle g(\xi),x - c\rangle$ is $\mu$-strongly convex and $\mu$-smooth, and its gradient belongs to $\mathcal{G}$ (by translation invariance of $\mathcal{G}$, see Assumption 1). Note that we have only added a constant term \wrt $x$ which is thus invisible to algorithms that rely on the gradient. Thus, the constant $c$ has no impact on the algorithm's output $x_{\sA[\sO]}$, and, for $\mathcal{H}_2 = \{\xi\mapsto \ell(x_{\sA[\sO]},\xi)~:~\mathbb{E}[\ell(\cdot,\xi)]\in\mathcal{F}_{sc}(\mathcal{G},\mathcal{D},\mu,L)\}$, we have, as soon as $\mathbb{E}[g(\xi)] \neq \mathbb{E}[g(\xi')]$,
\BEQ
d_{\mathcal{H}_2}(\mathcal{D},\mathcal{D}') \geq \sup_{c\in\mathbb{R}^d}|\mathbb{E}[\ell^{g,c}(xx_{\sA[\sO]},\xi) - \ell^{g,c}(x_{\sA[\sO]},\xi')]| = \sup_{c\in\mathbb{R}^d} |\langle\mathbb{E}[g(\xi)] - \mathbb{E}[g(\xi')],x_{\sA[\sO]} - c\rangle| = +\infty\,.
\EEQ
The same result holds if one replaces $x_{\sA[\sO]}$ by $\arg\min_x \mathbb{E}[\ell(x,\xi')]$, $\arg\min_x \mathbb{E}[\ell(x,\xi)]$, or any set independent of $c$.

\section{Proofs of \Sec{gen_bounds}}

\subsection{Proof of our lower bound}

We start by proving our first result (\Proposition{lb_GL}), the lower bound on the minimax excess risk that exhibits $\lb$ as limiting factor. This result is a direct consequence of the following lemma:
\begin{lemma}\label{lem:lb_GEM}
Let $\varepsilon> 0$, $\cG\subset\bF(\Xi,\R^d)$ be a function space and $\sO$ a data-dependent oracle verifying \Assumption{O}. Then, for any optimization algorithm $\sA\in\Arand$, there exists an objective function $\cL\in\cF_{\mbox{\tiny sc}}(\cG,\cD, \mu, L)$ such that
\BEQe
\Exp{\cL(x_{\sA[\sO]}) - \inf_{x\in\R^d} \cL(x)} \geq \frac{\lb^2}{2\mu} - \varepsilon\,.
\EEQe
\end{lemma}

\begin{proof}
For any $g\in\cG$, let
\BEQ\label{eq:lh}
\ell^g(x,\xi) = \frac{\mu}{2}\|x\|^2 + \scprod{x}{g(\xi)}\,.
\EEQ
First, note that $\ell^g$ is $\mu$-strongly convex and $\mu$-smooth \wrt $x$ (and $\mu \leq L$), and $\nabla_x \ell^g(x,\xi) = \mu x + g(\xi) \in \cG$ by stability of $\cG$ by translation.
By \Assumption{O}, there exists a measurable function $\varphi:\cG\times\R^d\to\cO$ such that $\sO(\nabla\ell^g_{x_{\sA[\sO]}^{(t)}}, z) = \varphi(\sO(g, z), \mu x_{\sA[\sO]}^{(t)})$, and 
we now show that the output of the algorithm $x_{\sA[\sO]}$ is a measurable function of $r$ and $\sO(g, z)$.

\begin{lemma}
For any optimization algorithm $\sA$ of \Definition{alg}, there exists a function $\psi_\sA$ such that the ouptut $x_{\sA[\sO]}$ of $\sA$ applied to any objective function $\ell^g$ defined in \Eq{lh} for $g\in\cG$ is
\BEQ
\mu\,x_{\sA[\sO]} \quad=\quad \psi_\sA(r,\sO(g,z))\,.
\EEQ
\end{lemma}

\begin{proof}
First, note that $\mu x_{\sA[\sO]}^{(0)} = \mu q^{(0)}(r)$ is a measurable function of $r$.
By induction over $t\geq 0$, there exists measurable functions $\psi_{x,\sA}^{(t)}:\cR\times\cO\to\R^d$ such that $\mu x_{\sA[\sO]}^{(t)} = \psi_{x,\sA}^{(t)}(r,\sO(g, z))$ and $\psi_{s,\sA}^{(t)}:\cR\times\cO\to\{0,1\}$ such that $s_{\sA[\sO]}^{(t)} = \psi_{s,\sA}^{(t)}(r,\sO(g, z))$. 
Without loss of generality, we assume that $s_{\sA[\sO]}^{(t)} = 1$ only once, as we can replace $s_{\sA[\sO]}^{(t)}$ on all iterations after the first $1$ by $0$.
Thus, we have
\BEQe
\mu\,x_{\sA[\sO]} \quad=\quad \mu \sum_{t=0}^{+\infty} s_{\sA[\sO]}^{(t)} x_{\sA[\sO]}^{(t)}\quad=\quad \sum_{t=0}^{+\infty} \psi_{s,\sA}^{(t)}(r,\sO(g, z))\psi_{x,\sA}^{(t)}(r,\sO(g, z))\,,
\EEQe
that is a measurable function of $r$ and $\sO(g, z)$ as a limit of measurable functions.
\end{proof}

Moreover, we have $\nabla \cL^g(x) = \mu x + \Exp{g(\xi)}$, and $\inf_{x\in\R^d} \cL^g(x) = -\|\Exp{g(\xi)}\|^2/2\mu$.
This gives $\cL^g(x_{\sA[\sO]}) - \inf_{x\in\R^d} \cL^g(x) = \|\mu x_{\sA[\sO]} + \Exp{g(\xi)}\|^2/2\mu$.
Now,
\begin{align*}
\sup_{g\in\cG} \Exp{\cL^g(x_{\sA[\sO]}) - \inf_{x\in\R^d} \cL^g(x)} =& \sup_{g\in\cG} \frac{\Exp{\|-\psi_\sA(r,\sO(g,z)) - \sE_\cD(g)\|^2}}{2\mu}\\
\geq& \frac{\lb^2}{2\mu}\,,
\end{align*}
where the last inequality follows from Jensen's inequality on $r$ and the definition of $\lb$.
\end{proof}

\subsection{General upper-bound and subsequent corollaries}

\begin{lemma}\label{lem:ub_GEM}
Let $\varphi\in\bM(\cO,\R^d)$ and $\cL\in\cF_{\mbox{\tiny pl}}(\cG,\cD, \mu, L)$. Then, the iterates $x_0 = 0$ and $x_{t+1} = x_t - \frac{1}{L} \varphi(o_t)$ where $o_t = \sO(\nabla\ell_{x_t}, z)$ achieve an approximation error
\BEQe
\Exp{\cL(x_t) - \inf_{x\in\R^d} \cL(x)} \leq \Delta \rho^t + \frac{\|\sE_\cD - \varphi\circ\sO\|_{2,\cG}^2}{2\mu}\,,
\EEQe
where $\rho = 1 - \mu/L$ and $\Delta = \cL(x_0) - \inf_{x\in\R^d} \cL(x)$.
\end{lemma}

\begin{proof}
First, recall that, for any $x\in\R^d$, $\nabla\ell_x\in\cG$ and $\ExpUnder{\xi\sim\cD}{\nabla_x\ell(x,\xi)} = \sE_\cD(\nabla\ell_x)$. Thus,
\BEQe
\BA{lll}
\Exp{\|\nabla \cL(x_t) - \varphi(o_t)\|^2} &=& \Exp{\|\sE_\cD(\nabla\ell_{x_t}) - \varphi\circ\sO(\nabla\ell_{x_t}, z)\|^2}\\
&\leq& \Exp{\sup_{g\in\cG}\|\sE_\cD(g) - \varphi\circ\sO(g, z)\|^2}\\
&=& \|\sE_\cD - \varphi\circ \sO\|_{2,\cG}^2
\EA
\EEQe
Then, by smoothness, we have
\BEQe
\BA{lll}
\cL(x_{t+1}) - \cL(x_t) &\leq& -\frac{1}{L} \scprod{\nabla \cL(x_t)}{\varphi(o_t)} + \frac{1}{2L}\|\varphi(o_t)\|^2\\
&=& -\frac{1}{2L}\|\nabla \cL(x_t)\|^2 + \frac{1}{2L}\|\nabla \cL(x_t) - \varphi(o_t)\|^2
\EA
\EEQe
Moreover, as $\cL$ is $\mu$-PL, we have
\BEQe
\|\nabla \cL(x_t)\|^2 \geq 2\mu \left(\cL(x_t) - \inf_{x\in\R^d} \cL(x)\right)\,.
\EEQe
Combining the two previous equations and taking the expectation gives
\BEQe
\Exp{\cL(x_{t+1}) - \cL(x_t)} \leq -\frac{1}{\kappa}\Exp{\cL(x_t) - \inf_{x\in\R^d} \cL(x)} + \frac{ \|\sE_\cD - \varphi\circ \sO\|_{2,\cG}^2}{2L}\,.
\EEQe
A simple recurrence gives
\BEQe
\Exp{\cL(x_t) - \inf_{x\in\R^d} \cL(x)} \leq \left(1 - \frac{1}{\kappa}\right)^t \Exp{\cL(x_0) - \inf_{x\in\R^d} \cL(x)} +  \frac{ \|\sE_\cD - \varphi\circ \sO\|_{2,\cG}^2}{2\mu}\,.
\EEQe
\end{proof}

\begin{proof}[\textbf{Proof of \Proposition{ub_GL}}]
We need to select a number of steps in \Lemma{ub_GEM} sufficient to reduce the first term in $\Delta (1-1/\kappa)^t$ to any given precision $\varepsilon > 0$. After the first iteration, we fix the number of iterations as $T_z=\kappa\ln\left(\frac{\|\varphi(o_0)\|^2 + \|\sE_\cD-\varphi\circ\sO\|_{2,\cG}^2}{\varepsilon\mu}\right)$. Note that this stopping time depends only on the observation $o_0$ at the first iteration, and can thus be computed after this iteration. Then, we have

\BEQe
\begin{aligned}
   \Exp{\cL(x_{T_z}) - \inf_{x\in\R^d} \cL(x)} &\leq \Exp{\left( 1 - \frac{1}{\kappa} \right)^{T_z} (\cL(x_0) - \inf_{x\in\R^d} \cL(x))}\\
   &\quad + \Exp{\frac{1}{2L}\sum_{t=0}^{T_z-1} \left( 1 - \frac{1}{\kappa} \right)^{T_z-t-1}\|\nabla \cL(x_t) - \varphi(o_t)\|^2}\\
   &\leq \varepsilon\mu\Exp{\frac{\cL(x_0) - \inf_{x\in\R^d} \cL(x)}{\|\varphi(o_0)\|^2 + \|\sE_\cD-\varphi\circ\sO\|_{2,\cG}^2}}\\
   &\quad + \Exp{\frac{1}{2L}\sum_{t=0}^{T_z-1} \left( 1 - \frac{1}{\kappa} \right)^{T_z-t-1}\sup_{g\in\cG}\|\sE_\cD(g) - \varphi\circ\sO(g,z)\|^2}\\
   &\leq 2\varepsilon + \frac{\Exp{\sup_{g\in\cG}\|\sE_\cD(g) - \varphi\circ\sO(g,z)\|^2}}{2\mu}\\
   &= 2\varepsilon + \frac{\|\sE_\cD - \varphi\circ\sO\|_{2,\cG}^2}{2\mu}\,,
\end{aligned}
\EEQe
where the last inequality follows from the $\mu$-PL condition and $\cL(x_0) - \inf_{x\in\R^d} \cL(x) \leq \|\nabla\cL(x_0)\|^2/2\mu \leq (2\|\varphi(o_0)\|^2 + 2\|\sE_\cD(\nabla\ell_{x_0}) - \varphi\circ\sO(\nabla\ell_{x_0},z)\|^2)/2\mu$ (see the beginning of the proof in \Lemma{ub_GEM}), leading to
\BEQe
\begin{aligned}
\Exp{\frac{\cL(x_0) - \inf_{x\in\R^d} \cL(x)}{\|\varphi(o_0)\|^2 + \|\sE_\cD-\varphi\circ\sO\|_{2,\cG}^2}} &\leq \frac{1}{2\mu}\Exp{\frac{2\|\varphi(o_0)\|^2 + 2\|\sE_\cD(\nabla\ell_{x_0}) - \varphi\circ\sO(\nabla\ell_{x_0},z)\|^2}{\|\varphi(o_0)\|^2 + \|\sE_\cD-\varphi\circ\sO\|_{2,\cG}^2}}\\
&\leq \frac{1}{\mu}\Exp{1 + \frac{\sup_{g\in\cG}\|\sE_\cD(g) - \varphi\circ\sO(g,z)\|^2}{\|\sE_\cD-\varphi\circ\sO\|_{2,\cG}^2}}\\
&= \frac{2}{\mu}\,.
\end{aligned}
\EEQe

Finally, taking $\varphi\in\bM(\cO,\R^d)$ such that $\|\sE_\cD-\varphi\circ\sO\|_{2,\cG} \leq \sigma_{2,\cG}(\sE_\cD|\sO) + \varepsilon$ and $\varepsilon\to 0$ gives the desired result.
\end{proof}

\begin{proof}[\textbf{Proof of \Corollary{deterministic}}]
If $\sO(g,z) = \tilde{\sO}(g)$ is independent of $z$, then $\lb = \ub = \sigma_\cG(\sE_\cD|\tilde{\sO})$ as all norms are equal, and \Proposition{ub_GL} immediately gives the desired result.
\end{proof}

\section{Proofs of \Sec{iid}}

Recall that for the following proofs, the oracle is assumed to be of the form $\sON{}$ ($n$ \iid observations).

\begin{lemma}\label{lem:mini-batch}
For any $n\geq 1$, let $\varphi_n$ be such that $\|\sE_\cD - \varphi_n\circ\sO_n\|_{\cG,2}^2 \leq \sigma_{\cG,2}\left( \sE_\cD | \sO_n \right)^2 + \varepsilon$, and $(n_1,\dots,n_T)$ be non-negative integers such that $\sum_{t < T} n_t \leq n$. Then, the iterates $x_0 = 0$ and
\BEQ
x_{t+1} = x_t - \frac{1}{L} \varphi_{n_t}\left(o_{t,1},\dots,o_{t,n_t}\right)\,,
\EEQ
where $o_{t,k} = \sO(\nabla\ell_{x_t}, z^{(N_t+k)})$ is a (fresh) \iid observation and $N_t = \sum_{i < t} n_i$, achieve after $T$ iterations an approximation error
\BEQ
\Exp{\cL(x_T) - \cL^*} \leq \Delta \rho^T + \sum_{t=0}^{T-1} \frac{\sigma_t^2\rho^{T-t-1}}{2L} + \frac{\varepsilon}{2\mu}\,,
\EEQ
where $\sigma_t = \sigma_{\cG,2}\left( \sE_\cD | \sO_{n_t} \right)$, $\cL^* = \inf_{x\in\R^d} \cL(x)$, $\rho = 1 - \mu/L$ and $\Delta = \cL(x_0) - \cL^*$.
\end{lemma}

\begin{proof}
As in Lemma \ref{lem:ub_GEM}, we have, using smoothness and the $\mu$-PL condition:
\BEQe
\cL(x_{t+1}) - \cL(x_t) \leq -\frac{1}{\kappa}\left(\cL(x_t) - \inf_{x\in\R^d} \cL(x)\right) + \frac{1}{2L}\|\nabla \cL(x_t) - \varphi_{n_t}(o_t)\|^2\,.
\EEQe
where $o_t=\sO(\nabla\ell_{x_t}, z^{(N_t+1)}), \dots, \sO(\nabla\ell_{x_t}, z^{(N_t+n_t)})$. Thus, by a simple recursion:
\BEQe
   \cL(x_T) - \inf_{x\in\R^d} \cL(x) \leq \Delta \left( 1 - \frac{1}{\kappa} \right)^T + \frac{1}{2L}\sum_{t=0}^{T-1} \left( 1 - \frac{1}{\kappa} \right)^{T-t-1}\|\nabla \cL(x_t) - \varphi_{n_t}(o_t)\|^2 \,.
\EEQe
To conclude, we take the expectation on both sides of the inequality: since $\varphi_{n_t}(o_t)$ is independent from $(z^{(i)})_{1\leq i \leq N_t}$ and $\nabla \cL(x_t)$ only depends on $(z^{(i)})_{1\leq i \leq N_t}$, $\nabla \cL(x_t)$ and $\varphi_{n_t}(o_t)$ are independent and thus $\E\|\nabla \cL(x_t) - \varphi_{n_t}(o_t)\|^2\leq \sigma_{\cG,2}\left( \sE_\cD | \sO_{n_t} \right)^2+\varepsilon$, by definition of $\varphi_{n_t}$.
\end{proof}

\begin{proof}[\textbf{Proof of \Proposition{iid}}]
We first start with a warm-up phase in which we only use the first observation $\sO(\nabla\ell_{x_t}, z^{(1)})$ for a number of steps sufficient to reduce the first term in Lemma \ref{lem:ub_GEM} in $\Delta (1-1/\kappa)^t$ to any given precision $\varepsilon > 0$. We then fix the number of iterations of this warm-up phase (after the first iteration) as $T_z=\kappa\ln\left(\frac{\|\varphi(o_0)\|^2 + \|\sE_\cD-\varphi\circ\sO_1\|_{2,\cG}^2}{\varepsilon\mu}\right)$. This gives
\BEQe
   \Exp{\cL(x_{T_z}) - \inf_{x\in\R^d} \cL(x)} \quad\leq\quad 2\varepsilon + \frac{\|\sE_\cD - \varphi\circ\sO_1\|_{2,\cG}^2}{2\mu} \quad\leq\quad 2\varepsilon + \frac{\sigma_{2,\cG}(\sE_\cD|\sO_1)^2 + \varepsilon}{2\mu}\,,
\EEQe
for a function $\varphi\in\bM(\cO,\R^d)$ such that $\|\sE_\cD-\varphi\circ\sO_1\|_{2,\cG}^2 \leq \sigma_{2,\cG}(\sE_\cD|\sO_1)^2 + \varepsilon$.
Then, we apply \Lemma{mini-batch} starting at $x_0'=x_{T_z}$ with a number of steps $T=\lceil a\kappa \log n \rceil$ and a fixed mini-batch size of $n_t = N = \left\lfloor\frac{n-1}{1 + a\kappa\log n}\right\rfloor$. This ensure that $\sum_{t < T} n_t \leq n - 1$ and gives
\BEQ
\BA{lll}
\Exp{\cL(x_{T+T_z}) - \inf_{x\in\R^d} \cL(x)} &\leq& \widetilde{\Delta} \left( 1 - \frac{1}{\kappa} \right)^T + \frac{\sigma_{\cG,2}\left( \sE_\cD | \sO_N \right)^2}{2\mu} + \frac{\varepsilon}{2\mu}\\
&\leq& \widetilde{\Delta} e^{-T/\kappa} + \frac{\sigma_{\cG,2}\left( \sE_\cD | \sO_N \right)^2}{2\mu} + \frac{\varepsilon}{2\mu}\\
&\leq& \widetilde{\Delta} n^{-a} + \frac{\sigma_{\cG,2}\left( \sE_\cD | \sO_N \right)^2}{2\mu} + \frac{\varepsilon}{2\mu}\,,
\EA
\EEQ
where $\widetilde{\Delta}=\Exp{\cL(x_{T_z}) - \inf_{x\in\R^d} \cL(x)} \leq 2\varepsilon + \frac{\sigma_{2,\cG}(\sE_\cD|\sO_1)^2 + \varepsilon}{2\mu}$. Finally, letting $\varepsilon$ tend to $0$ concludes the proof.
\end{proof}

\begin{proof}[\textbf{Proof of \Proposition{iid_increasing_batch}}]
Similarly to the proof of \Proposition{iid}, we start with a warm-up phase of $T_z=\kappa\ln\left(\frac{\|\varphi(o_0)\|^2 + \|\sE_\cD-\varphi\circ\sO_1\|_{2,\cG}^2}{\varepsilon\mu}\right)$ steps using only the first observation $\sO(\nabla\ell_{x_t}, z^{(1)})$.
This gives, for a function $\varphi\in\bM(\cO,\R^d)$ such that $\|\sE_\cD-\varphi\circ\sO_1\|_{2,\cG}^2 \leq \sigma_{2,\cG}(\sE_\cD|\sO_1)^2 + \varepsilon$, an approximation error
\BEQe
   \Exp{\cL(x_{T_z}) - \inf_{x\in\R^d} \cL(x)} \quad\leq\quad 2\varepsilon + \frac{\|\sE_\cD - \varphi\circ\sO_1\|_{2,\cG}^2}{2\mu} \quad\leq\quad 2\varepsilon + \frac{\sigma_{2,\cG}(\sE_\cD|\sO_1)^2 + \varepsilon}{2\mu}\,.
\EEQe
We then apply \Lemma{mini-batch} starting at $x_0'=x_{T_z}$ with a number of steps $T=\lfloor (n-1)/2 \rfloor$ and an increasing mini-batch of size $n_t = \left\lceil (n-1) (1-c)c^{T-t-1}/2 \right\rceil$ where $c = \sqrt{1-\kappa^{-1}}$. This ensures that $\sum_{t < T} n_t \leq \sum_{t < T} (1 + (n-1) (1-c)c^{T-t-1}/2) \leq T + \frac{(n-1)(1-c)}{2(1-c)} \leq \frac{n-1}{2} + \frac{n-1}{2} = n-1$. Thus, applying \Lemma{mini-batch} gives
\BEQ
\BA{lll}
\Exp{\cL(x_T) - \inf_{x\in\R^d} \cL(x)} &\leq& \widetilde{\Delta} c^{2T} + \frac{1}{2L}\sum_{t=0}^{T-1} c^{2(T-t-1)}\left( a + \frac{b}{n_t} \right) + \frac{\varepsilon}{2\mu}\\
&\leq& \widetilde{\Delta} e^{-T/\kappa} + \frac{1}{2L}\sum_{t=0}^{T-1} \frac{2b c^{T-t-1}}{(n-1)(1-c)} + \frac{a+\varepsilon}{2\mu}\\
&\leq& \widetilde{\Delta} e^{-\frac{n-2}{2\kappa}} + \frac{b}{(n-1)L(1-c)^2} + \frac{a+\varepsilon}{2\mu}\,,
\EA
\EEQ
where $\widetilde{\Delta}=\Exp{\cL(x_{T_z}) - \inf_{x\in\R^d} \cL(x)} \leq 2\varepsilon + \frac{\sigma_{2,\cG}(\sE_\cD|\sO_1)^2 + \varepsilon}{2\mu}$. We conclude by noting that $L(1-c)^2 = \mu \kappa(1-\sqrt{1-\kappa^{-1}})^2 \geq \mu\kappa(2\kappa)^{-2} = \mu/4\kappa$ and, if $n\geq 3$, we have $n-2\geq n/3$ and $n-1\geq 2n/3$.
\end{proof}

\section{Proofs of \Sec{applications}}

\subsection{Empirical risk minimization}

\begin{proof}[\textbf{Proof of \Proposition{upper_SL}}]
Let $\varphi(g_1,\dots,g_n) = \frac{1}{n}\sum_i g_i$ be the average over the $n$ data points, then we directly have $\|\sE_\cD - \varphi\circ\sON{SL}\|_{\cG,2}^2 = \sup_{g\in\cG} \Exp{\|\frac{1}{n} \sum_i (g(\xi_i) - \ExpUnder{\xi\sim\cD}{g(\xi)})\|^2} = \sup_{g\in\cG} \var(g(\xi_1))/n$ as the data samples $\xi_i$ are \iid random variables.
\end{proof}

We next prove the upper bounds on $\SV{\cD}{\cG} = \sup_{g\in\cG} \var(g(\xi_1))$ provided after \Proposition{upper_SL} in the main text.

\begin{proof}[Proof of the upper bounds on $\SV{\cD}{\cG}$]
We begin with affine functions. Let $(g:\xi\mapsto A\xi + b)\in\cGN{Aff}$. We have $\var(g(\xi))=\var(A\xi)=\esp{\NRM{A(\xi-\E\xi)}^2}\leq B^2 \esp{\NRM{\xi-\E\xi}}^2=B^2\var(\xi)$, since $\rho(A)\leq B$. By taking a supremum over $g\in\cGN{Aff}$, we have the desired result.
The equality is obtained by taking any rank $D$ projection for $A$, in the case $D\leq d$.

For Lipschitz functions: let $g\in\cGN{Lip}$. We have $\var(g(\xi))=\esp{\NRM{g(\xi)-\E g(\xi)}^2}=\frac{1}{2}\esp{\NRM{g(\xi)- g(\xi')}^2}$, where $\xi'\sim \cD$ is independent from $\xi$. Thus, using the Lipschitzness of $g$, $\var(g(\xi))\leq B^2\frac{1}{2}\esp{\NRM{\xi- \xi'}^2}=B^2\var(\xi)$, and we take the supremum over $g$.

For functions with bounded variations: let $g\in\cGN{Bnd}$. There exists $c\in\R^d$ such that for all $\xi$, $\NRM{g(\xi)-c}\leq B$.
Using $\var(g(\xi))=\esp{\NRM{g(\xi)-\Exp{g(\xi)}}^2}\leq\esp{\NRM{g(\xi)- c}^2}\leq B^2$, we have $\var(g(\xi))\leq B^2$, and we take the supremum over $g$.
For the equality, we take $g$ such that $g(\xi)=(B,0,\ldots,0)$ for $\xi\in A$ and $g(\xi)=(-B,0,\ldots,0)$ for $\xi\notin A$.
\end{proof}

We finally specify our SL results for $\cGN{Bnd}$ with the example or regularized SL.

\begin{lemma}\label{lem:sigma_bnd}
Assume that $\forall p\in[0,1]$, $\exists A\subset\Xi$ measurable \st $\ProbUnder{\cD}{A} = p$. Then, for $n\geq 1$, we have
\BEQ
\frac{B}{1 + \sqrt{n}} \,\leq\, \sigma_{\cGN{Bnd},2}(\sE_\cD|\sON{SL}) \,\leq\, \frac{B}{\sqrt{n}}\,.
\EEQ
Moreover, the average $\varphi(x) = \frac{1}{n}\sum_i x_i$ is asymptotically optimal, as $\|\sE_\cD - \varphi\circ\sON{SL}\|_{\cGN{Bnd},2}^2 = B^2/n$.
\end{lemma}

\begin{proof}
First, the upper bound is a direct application of \Proposition{upper_SL} for $\cGN{Bnd}$, and is achieved for the average $\varphi(x) = \frac{1}{n}\sum_i x_i$.
To prove the lower bound, we replace the supremum over all functions in $\cGN{Bnd}$ by functions of the form $g:\Xi\to\{-Bv,Bv\}$ where $\|v\|=1$.
\begin{align*}
\sigma_{\cGN{Bnd},2}(\sE_\cD|\sON{SL})^2
&\geq \inf_{\varphi\in\bM(\cO,\R^d)} \sup_{g\in\bF(\Xi,\{-Bv,Bv\})} \Exp{\| 
\sE_\cD(g) - \varphi(g(\xi_1),\dots,g(\xi_n)) \|^2}\\
&\geq B^2 \inf_{\varphi\in\bM(\cO,\R^d)} \sup_{p\in[0,1]} \Exp{\left(
2p - 1 - \varphi(G_1,\dots,G_n)\right)^2}\,,
\end{align*}
where $G_i=\one\{g(\xi_i) = Bv\}$ are \iid Bernoulli random variables of parameter $p=\Prob{g(\xi_1) = Bv}$. We now replace the probability $p$ by a random variable $P\sim\mbox{\rm Beta}(a,b)$ for $a,b>0$, which gives
\begin{align*}
\sigma_{\cGN{Bnd},2}(\sE_\cD|\sON{SL})^2
&\geq B^2 \inf_{\varphi\in\bM(\cO,\R^d)} \Exp{\left(
2P - 1 - \varphi(G_1,\dots,G_n)\right)^2}\\
&= 4B^2 \inf_{\varphi\in\bM(\cO,\R^d)} \Exp{\left(
P - \varphi(G_1,\dots,G_n)\right)^2}\\
&= 4B^2\, \Exp{\left(
P - \Exp{P~|~G_1,\dots,G_n}\right)^2}\\
&= 4B^2\, \Exp{\left(
P - \Exp{P~|~K}\right)^2}\,,
\end{align*}
where $K=\sum_i G_i$ is a Binomial distribution of parameters $n$ and $p$, as the Bernoulli r.v. are identically distributed. A simple calculation gives that $P|K\sim\mbox{\rm Beta}(K+a,n-K+b)$, which allows us to compute the quantities $\Exp{P^2} = \frac{ab}{(a+b)^2(1+a+b)} + \frac{a^2}{(a+b)^2}$ and $\Exp{P~|~K} = \frac{K+a}{n+a+b}$. We thus obtain
\begin{align*}
\sigma_{\cGN{Bnd},2}(\sE_\cD|\sON{SL})^2
&\geq 4B^2\, \left(\Exp{
P^2} - \Exp{\Exp{P~|~K}^2}\right)\\
&= 4B^2\, \left(\frac{ab}{(a+b)^2(1+a+b)} + \frac{a^2}{(a+b)^2} - \Exp{\left(\frac{K+a}{n+a+b}\right)^2}\right)\\
&= \frac{4B^2 ab}{(a+b)(1+a+b)(n+a+b)}\,,
\end{align*}
where the last equality is obtained using $\Exp{(K+a)^2} = \Var(K) + (\Exp{K} + a)^2 = \frac{nab(n+a+b)}{(a+b)^2(1+a+b)} + \left(\frac{a(n+a+b)}{a+b}\right)^2$. Finally, choosing $a=b=\sqrt{n}/2$ gives the desired result.
\end{proof}

\begin{proof}[\textbf{Proof of \Proposition{bounded}}]
Noticing that $\sigma_{\cG,2}\left( \sE_\cD | \sO_n \right)^2 \leq B^2/n$ and $\sigma_{2,\cG}\left( \sE_\cD | \sO_1 \right)^2 \leq 4B^2$, we can use \Proposition{iid_increasing_batch} with $a=0$ in order to obtain
\BEQe
\mmratesc{\cGN{Bnd},\sON{SL}} \quad\leq\quad \frac{6\kappa B^2}{\mu n}+ \frac{2 B^2}{\mu} e^{-\frac{n}{6\kappa}}\,.
\EEQe
The right handside of \Proposition{bounded} is obtained by using $e^{-\frac{n}{6\kappa}} \leq \frac{6\kappa}{ne}$ and $6+12/e=10.41...\leq 11$.
The left handside of the desired inequality is then proved using the lower bound of \Lemma{sigma_bnd} together with the general lower bound of \Lemma{lb_GEM}.
\end{proof}

\subsection{Transfer learning}

\begin{proof}[\textbf{Proof of \Proposition{TLrad}}]
Let $\varphi(g_1,\dots,g_n) = \frac{1}{n}\sum_i g_i$ be the average over the $n$ data points, then we directly have, for $\xi\sim\cD$,
\BEQ
\BA{lll}
\|\sE_\cD - \varphi\circ\sON{TL}\|_{\cG,2}^2 &=& \sup_{g\in\cG} \Exp{\|\Exp{g(\xi)} - \frac{1}{n} \sum_i g(\xi'_i)\|^2}\\
&=& \sup_{g\in\cG} \|\Exp{g(\xi)} -\Exp{g(\xi'_1)}\|^2 + \frac{\var(g(\xi'_1))}{n}\\
&\leq& \sup_{g\in\cG} \|\Exp{g(\xi)} -\Exp{g(\xi'_1)}\|^2 + \sup_{g\in\cG} \frac{\var(g(\xi'_1))}{n}\\
&=& d_\cG(\cD,\cD')^2 + \frac{\|\sV_{\cD'}\|_\cG}{n}\,.\\
\EA
\EEQ
\end{proof}

\begin{proof}[\textbf{Proof of \Proposition{TLBnd}}]
By assumption, $\forall\varepsilon\in(0,1], \exists q_\varepsilon\in\R$ \st $\ProbUnder{\cD'}{\frac{d\cD}{d\cD'}(\xi') \geq q_\varepsilon} = (1+\varepsilon)/2$. We apply \Proposition{lecam} to $\sON{TN}$ with $g(\xi)=-g'(\xi)=B\left(2\one\left\{\frac{d\cD}{d\cD'}(\xi) \geq q_\varepsilon\right\} - 1\right)e_1$ where $e_1=(1,0,\dots)^\top$ is the first basis vector. First, note that $g,g'\in\cGN{Bnd}$, which gives
\BEQ
\lb^2 \geq \left(1 - \dN{LC}(\sON{TL}(g,z), \sON{TL}(-g,z))\right)\,\|\sE_\cD(g)\|^2\,,
\EEQ
where $\dN{LC}(p,q)$ is Le Cam's distance (see \Appendix{lecam}). By definition,
\BEQ
\BA{lll}
\|\sE_\cD(g)\|^2 &=& B^2 \ExpUnder{\cD}{2\one\left\{\frac{d\cD}{d\cD'}(\xi') \geq q_\varepsilon\right\} - 1}^2\\
&=& B^2 \ExpUnder{\cD'}{\left(2\one\left\{\frac{d\cD}{d\cD'}(\xi') \geq q_\varepsilon\right\} - 1 \right)\frac{d\cD}{d\cD'}(\xi')}^2\\
&=& B^2 \left(\ExpUnder{\cD'}{\left(2\one\left\{\frac{d\cD}{d\cD'}(\xi') \geq q_\varepsilon\right\} - 1 \right)\left(\frac{d\cD}{d\cD'}(\xi') - q_\varepsilon\right)} + q_\varepsilon\varepsilon \right)^2\\
&=& B^2 \left(\ExpUnder{\cD'}{\left|\frac{d\cD}{d\cD'}(\xi') - q_\varepsilon\right|} + q_\varepsilon\varepsilon \right)^2\,.
\EA
\EEQ
Moreover, we have $\ExpUnder{\cD'}{\left|\frac{d\cD}{d\cD'}(\xi') - q_\varepsilon\right|} \geq 2\dN{TV}(\cD,\cD') - |q_\varepsilon - 1|$ and $\ExpUnder{\cD'}{\left|\frac{d\cD}{d\cD'}(\xi') - q_\varepsilon\right|} \geq |q_\varepsilon - 1|$. As a consequence, we have $\ExpUnder{\cD'}{\left|\frac{d\cD}{d\cD'}(\xi') - q_\varepsilon\right|} \geq \dN{TV}(\cD,\cD')$ and $\ExpUnder{\cD'}{\left|\frac{d\cD}{d\cD'}(\xi') - q_\varepsilon\right|} + q_\varepsilon\varepsilon \geq |q_\varepsilon - 1| + q_\varepsilon\varepsilon \geq \varepsilon$, which gives
\BEQ
\|\sE_\cD(g)\|^2 \geq B^2 \max\left\{\dN{TV}(\cD,\cD'), \varepsilon\right\}^2 \geq \frac{B^2}{2} \left(\dN{TV}(\cD,\cD')^2 + \varepsilon^2\right)\,.
\EEQ
Finally, we conclude by noting that, as $g$ only takes two values ($-Be_1$ and $Be_1$), we have that, if $N=|\{i\in\set{1,n}~|~g(\xi_i') = Be_1\}|$ and $\varepsilon \leq 1/2$, then
\BEQ
\BA{lll}
\dN{LC}(\sON{TL}(g,z), \sON{TL}(-g,z)) &\leq& \dN{KL}(\sON{TL}(g,z), \sON{TL}(-g,z))\\
&=& (2\Exp{N}-n)\ln\left( \frac{1+\varepsilon}{1-\varepsilon} \right)\\
&=& n\varepsilon\ln\left( \frac{1+\varepsilon}{1-\varepsilon} \right)\\
&\leq& 2\ln(3)n\varepsilon^2\,,
\EA
\EEQ
and taking $\varepsilon = 1/(2\sqrt{n})$ gives the desired result
\BEQ
\lb^2 \geq \frac{\left(2 - \ln(3)\right)B^2}{16} \left(\dN{TV}(\cD,\cD')^2 + \frac{1}{n}\right)\,.
\EEQ
\end{proof}

\subsection{(Personalized) Federated Learning}

We now prove the following upper bound:
\begin{equation*}
    \sigma_{\cG,2}(\sE_{\cD}|\sON{FL})^2 \quad\leq\quad \inf_{q\in\R^m}d_\cG\left(\cD,\sum_i q_i\cD_i\right)^2 + \sum_{i=1}^m q_i^2\frac{\SV{\cD_i}{\cG}}{n_i}\,.
\end{equation*}

\begin{proof}[\textbf{Proof of \Proposition{federated}}]
    We start by upper bounding using the function $\varphi((g^i_j)_{i\in\set{1,m},j\in\set{1,n_i}})=\sum_{i=1}^m\sum_{j=1}^{n_i} \frac{q_i}{n_i} g^i_j$, leading to
\begin{align*}
    \sigma_{\cG,2}(\sE_{\cD}|\sON{FL})^2&\leq \sup_{g\in\cG} \Exp{\NRM{\sE_\cD(g) - \sum_{i=1}^m\sum_{j=1}^{n_i} \frac{q_i}{n_i} g(\xi^i_j)  }^2}\\
    &=\sup_{g\in\cG} \esp{\NRM{\sE_\cD(g) - \Exp{\sum_{i=1}^m\sum_{j=1}^{n_i} \frac{q_i}{n_i} g(\xi^i_j)}    }^2}+\var\left(\sum_{i=1}^m\sum_{j=1}^{n_i} \frac{q_i}{n_i} g(\xi^i_j)\right)\\
    &=\sup_{g\in\cG} \esp{\NRM{\sE_\cD(g) - \sum_{i=1}^m q_i \sE_{\cD_i}(g) }^2}+\sum_{i=1}^m\frac{q_i^2}{n_i}\var\left( g(\xi^i_1)\right)\\
    &\leq d_\cG\left(\cD,\sum_i q_i \cD_i\right)^2 + \sum_i \frac{q_i^2}{n_i} \SV{\cD_i}{\cG}\,.
\end{align*}
We conclude by taking the infimum over $(q_i)$.
\end{proof}

\subsection{Robust Learning}

We now prove the two results of the robust learning section.
First, the upper and lower bound for $\cGN{Bnd}$,
\BEQ
\sigma_{\cGN{Bnd},2}(\sE_\cD|\sON{RL})^2 = \Theta\left(B^2 \left(\eta^2 + \frac{1}{n}\right)\right)\,,
\EEQ
is obtained by applying \Proposition{TLBnd} to $\cD' = (1-\eta)\cD + \eta\cD_o$, and noting that,
\BEQ
\dN{TV}(\cD,\cD') = \frac{1}{2}\sup_{f\in\bF(\Xi,[-1,1])} \sE_\cD(f) - \sE_{\cD'}(f) = \frac{\eta}{2}\sup_{f\in\bF(\Xi,[-1,1])} \sE_\cD(f) - \sE_{\cD_o}(f) = \eta\,,
\EEQ
as $\sE_{\cD'}(f) = (1-\eta)\sE_\cD(f) + \eta \sE_{\cD_o}(f)$ and, by assumption, $\dN{TV}(\cD,\cD_o) = 1$.
The second result, for $\cGN{Lip}$,
\BEQ\label{eq:RLproof}
\sigma_{\cGN{Lip},2}(\sE_\cD|\sON{RL})^2 \leq c B^2 \var(\xi) \left(\eta + \frac{1}{n}\right)\,,
\EEQ
is proved as follows: let $\varphi$ be the robust mean estimator of \citet[Algorithm 1]{SteinhardtCV18}. First, note that, for a set of data points $(\xi_1,\dots,\xi_m)$, we have, with $\lambda_{\max}(M)$ denoting the largest singular value of the symmetric matrix $M$,
\BEQ
\BA{lll}
\lambda_{\max}\left( \frac{1}{m} \sum_{i=1}^m (g(\xi_i) - \bar{g}_m)(g(\xi_i) - \bar{g}_m)^\top \right) &=& \max_{x~:~\|x\|\leq 1} \frac{1}{m} \sum_{i=1}^m (x^\top(g(\xi_i) - \bar{g}_m))^2\\
&\leq& \frac{1}{m} \sum_{i=1}^m \|g(\xi_i) - \bar{g}_m\|^2\\
&\leq& \frac{B^2}{m} \sum_{i=1}^m \|\xi_i - \bar{\xi}_m\|^2\,,
\EA
\EEQ
where $\bar{g}_m$ and $\bar{\xi}_m$ are, respectively, the averages of $g(\xi_i)$ and $\xi_i$ over all data points. Thus, assuming without loss of generality that we place all the outliers at the end of the sequence $(\xi_1,\dots,\xi_n)$, we can use $\sigma_0^2 = \frac{B^2}{(1-\eta)n} \sum_{i=1}^{(1-\eta)n} \|\xi_i - \bar{\xi}_{(1-\eta)n}\|^2$ in Proposition 16 of \citet{SteinhardtCV18}, and get that, if $\eta \leq 1/4$, the robust mean estimator always returns a value $\varphi\circ \sON{RL}(g,z)$ such that
\BEQ
\|\bar{g}_{(1-\eta)n} - \varphi\circ\sON{RL}(g,z)\|^2 \leq \frac{c^2 B^2}{(1-\eta)n} \sum_{i=1}^{(1-\eta)n} \|\xi_i - \bar{\xi}_{(1-\eta)n}\|^2 \eta
\EEQ
where $c=40$ is a universal constant. Finally, we take the expectation over the samples (i.e. $z$) and have
\BEQ
\BA{lll}
\|\sE_\cD - \varphi\circ \sON{RL}\|_{\cGN{Lip}}^2 &\leq& 2\sup_{g\in\cGN{Lip}} \Exp{\|\sE_\cD(g) - \bar{g}_{(1-\eta)n}\|^2 + \|\bar{g}_{(1-\eta)n} - \varphi\circ\sON{RL}(g,z)\|^2}\\
&\leq& \frac{2B^2\var(\xi)}{(1-\eta)n} + \frac{2c^2 B^2}{(1-\eta)n} \sum_{i=1}^{(1-\eta)n} \Exp{\|\xi_i - \bar{\xi}_{(1-\eta)n}\|^2}\eta\,.
\EA
\EEQ
We conclude by showing that $\Exp{\|\xi_i - \bar{\xi}_{(1-\eta)n}\|^2} = (1-\frac{1}{n})\var(\xi) \leq \var(\xi)$ and $\frac{1}{1-\eta} \leq 4/3$, giving \Eq{RLproof} with the constant $c=3200$.

\subsection{Learning with fixed data points}
 We now provide a proof of the upper and lower bounds for the minimax excess risk under the fixed data learning scenario.

\begin{proof}[\textbf{Proof of \Proposition{fixed_learn}}]
First, as $\sON{FD}$ is deterministic, \Corollary{deterministic} immediately gives that
\BEQ
\mmratesc{\cGN{Bnd},\sON{FD}} = \frac{\sigma_{2,\cGN{Bnd}}(\sE_\cD|\sON{FD})^2}{2\mu} \quad\mbox{and}\quad \mmratesc{\cGN{Lip},\sON{FD}} = \frac{\sigma_{2,\cGN{Lip}}(\sE_\cD|\sON{FD})^2}{2\mu}\,.
\EEQ
We thus only need to show that $\sigma_{2,\cGN{Bnd}}(\sE_\cD|\sON{FD}) = 2B\left(1 - \ProbUnder{\cD}{\{\xi_i'\}_{i\set{1,n}}}\right)$ and $\sigma_{2,\cGN{Lip}}(\sE_\cD|\sON{FD}) = B\Exp{\min_i\|\xi - \xi_i'\|}$ to conclude.

\textbf{Case $\cGN{Bnd}$:} First, let us assume that all $\xi_i$ are distinct, as we can otherwise remove the duplicates without any loss of generality. Then, consider the estimator $\varphi(g_1,\dots,g_n) = \sum_{i=1}^n \ProbUnder{\cD}{\{\xi_i'\}} g_i + (1-\ProbUnder{\cD}{\{\xi_i'\}_{i\set{1,n}}})\frac{1}{n}\sum_i g_i$, which can be used here as the $\xi_i'$ are fixed data points and $\cD$ is known. Thus, using $\varphi$, we have
\BEQ
\BA{lll}
\sigma_{2,\cGN{Bnd}}(\sE_\cD|\sON{FD}) &\leq& \sup_{g\in\cGN{Bnd}} \|\sE_\cD(g) - \varphi(g(\xi_1'),\dots,g(\xi_n'))\|\\
&=& \sup_{g\in\cGN{Bnd}} \|\Exp{\one\{\xi\notin\{\xi_i'\}_{i\in\set{1,n}}\} (g(\xi) - \frac{1}{n}\sum_i g(\xi_i'))}\|\\
&\leq& \sup_{g\in\cGN{Bnd}} \Exp{\one\{\xi\notin\{\xi_i'\}_{i\in\set{1,n}}\} \|g(\xi) - \frac{1}{n}\sum_i g(\xi_i')\|}
\\
&\leq& 2B\left(1 - \ProbUnder{\cD}{\{\xi_i'\}_{i\set{1,n}}}\right)\,.
\EA
\EEQ
The lower bound is obtained using \Proposition{lecam} with $g(\xi) = -g'(\xi) = \one\{\xi\notin\{\xi_i'\}_{i\in\set{1,n}}\}2B e_1 $, where $e_1$ is the first basis vector. This gives $\sON{FD}(g,z) = \sON{FD}(g',z) = (0,\dots,0)$, and thus $\dN{TV}(\sON{FD}(g,z),\sON{FD}(g',z)) = 0$ and
\BEQ
\sigma_{2,\cGN{Bnd}}(\sE_\cD|\sON{FD}) \geq \|\sE_\cD(g)\| = 2B\left(1 - \ProbUnder{\cD}{\{\xi_i'\}_{i\set{1,n}}}\right)\,,
\EEQ
which concludes the proof.

\textbf{Case $\cGN{Lip}$:} Again, we assume, without loss of generality, that all $\xi_i$ are distinct. For the upper bound, we use the estimator $\varphi(g_1,\dots,g_n) = \sum_{i=1}^n \ProbUnder{\cD}{\argmin_j \|\xi - \xi_j'\| = i} g_i$, which gives
\BEQ
\BA{lll}
\sigma_{2,\cGN{Lip}}(\sE_\cD|\sON{FD}) &\leq& \sup_{g\in\cGN{Lip}} \|\sE_\cD(g) - \varphi(g(\xi_1'),\dots,g(\xi_n'))\|\\
&=& \sup_{g\in\cGN{Lip}} \|\Exp{\sum_{i=1}^n \one\{\argmin_j \|\xi - \xi_j'\| = i\} (g(\xi) - g(\xi_i'))}\|\\
&\leq& \sup_{g\in\cGN{Lip}} \Exp{\sum_{i=1}^n \one\{\argmin_j \|\xi - \xi_j'\| = i\} \|g(\xi) - g(\xi_i')\|}\\
&\leq& B\,\Exp{\sum_{i=1}^n \one\{\argmin_j \|\xi - \xi_j'\| = i\} \|\xi - \xi_i'\|}\\
&=& B\,\Exp{\min_i\|\xi - \xi_i'\|}\,.
\EA
\EEQ
For the lower bound, we apply \Proposition{lecam} with $g(\xi) = -g(\xi) = B\min_i\|\xi - \xi_i'\|e_1$, which is $B$-Lipschitz by construction. As $\sON{FD}(g,z) = \sON{FD}(g',z) = (0,\dots,0)$, we have $\dN{TV}(\sON{FD}(g,z),\sON{FD}(g',z)) = 0$ and
\BEQ
\sigma_{2,\cGN{Lip}}(\sE_\cD|\sON{FD}) \geq \|\sE_\cD(g)\| = B\,\Exp{\min_i\|\xi - \xi_i'\|}\,,
\EEQ
which concludes the proof.
\end{proof}

\end{document}